\documentclass[11pt]{article}
\pdfoutput=1

\def\pb{}

\usepackage[dvipsnames]{xcolor}

\def\beq{\begin{equation} }\def\eeq{\end{equation} }\def\1{\mathbf{1}}

\usepackage[framemethod=default]{mdframed}
\usepackage{caption}
\usepackage{indentfirst}
\usepackage{bm, mathrsfs, graphics,float,amssymb,amsmath,subeqnarray,setspace,graphicx,amsthm,epstopdf,subfigure, enumerate, color}
\usepackage[utf8]{inputenc}
\usepackage[colorlinks,
      linkcolor=red,
      anchorcolor=blue,
      citecolor=blue
      ]{hyperref}
\usepackage{natbib}

\usepackage{fullpage}
\numberwithin{equation}{section}

\newtheorem{lemma}{Lemma}
\newtheorem{theorem}{Theorem}

\newtheorem{corollary}[theorem]{Corollary}
\newtheorem{remark}{Remark}

\ifx\assumption\undefined
\newtheorem{assumption}{Assumption}
\fi

\newcommand{\cO}{\mathcal{O}}

\newcommand{\EE}{\mathbb{E}}
\newcommand{\RR}{\mathbb{R}}

\newcommand{\ub}{\bm{u}}

\newcommand{\x}{\bm{x}}

\usepackage{bbm}

\def\cK{\mathcal{K}}

\newcommand{\db}{\bm{d}}
\newcommand{\Ub}{\bm{U}}
\newcommand{\wb}{\bm{w}}
\newcommand{\y}{\bm{y}}
\newcommand{\Ci}{C_{N, d, \delta}}
\newcommand{\Cd}{C_{d,\delta}}
\newcommand{\Ck}{N/2}
\newcommand{\cN}{\mathcal{N}}
\newcommand{\cE}{\mathcal{E}}

\usepackage{multirow}
\usepackage{tablefootnote}
\usepackage{colortbl}
\usepackage{hhline}

\usepackage{algorithm}
\usepackage{algorithmic}

\begin{document}
\title{
Explicit and Non-asymptotic Query Complexities of Rank-Based Zeroth-order Algorithms on Smooth Functions
}

\author{
	Haishan Ye
	\thanks{
	Xi'an Jiaotong University;
		email: hsye\_cs@outlook.com
	}
}
\date{
	\today}

\maketitle

\def\RB{\RR}
\def\TH{\tilde{H}}
\newcommand{\ti}[1]{\tilde{#1}}
\def\diag{\mathrm{diag}}
\newcommand{\norm}[1]{\left\|#1\right\|}
\newcommand{\dotprod}[1]{\left\langle #1\right\rangle}
\def\EB{\EE}
\def\tr{\mathrm{tr}}

\begin{abstract}
Rank-based zeroth-order (ZO) optimization---which relies only on the ordering of function evaluations---offers strong robustness to noise and monotone transformations, and underlies many successful algorithms such as CMA-ES, natural evolution strategies, and rank-based genetic algorithms. Despite its widespread use, the theoretical understanding of rank-based ZO methods remains limited: existing analyses provide only asymptotic insights and do not yield explicit convergence rates for algorithms selecting the top-$k$ directions.

This work closes this gap by analyzing a simple rank-based ZO algorithm and establishing the first \emph{explicit}, and \emph{non-asymptotic} query complexities. For a $d$-dimension problem, if the function is $L$-smooth and $\mu$-strongly convex, the algorithm achieves $\widetilde{\mathcal O}\!\left(\frac{dL}{\mu}\log\!\frac{dL}{\mu\delta}\log\!\frac{1}{\varepsilon}\right)$ to find an $\varepsilon$-suboptimal solution,
and for smooth nonconvex objectives it reaches $\mathcal O\!\left(\frac{dL}{\varepsilon}\log\!\frac{1}{\varepsilon}\right)$. 
Notation $\cO(\cdot)$  hides constant terms and $\widetilde{\mathcal O}(\cdot)$ hides extra $\log\log\frac{1}{\varepsilon}$ term.
These query complexities hold with a probability at least  $1-\delta$ with $0<\delta<1$. 
The analysis in this paper is novel and  avoids classical drift and information-geometric techniques. 
Our analysis offers new insight into why rank-based heuristics lead to efficient ZO optimization.
\end{abstract}


\pb\section{Introduction}

Zeroth-order optimization—optimization using only function evaluations—has become indispensable in a broad range of modern applications. Among these methods, \emph{rank-based} ZO algorithms hold a central and historically significant position. Unlike value-based or gradient-estimation-based methods, rank-based algorithms rely solely on the relative ordering of sampled candidate points. This makes them robust to noise, monotone transformations of the objective, and scaling issues, which in turn has led to widespread empirical success. 
Accordingly, a lots of rank-based zeroth-order algorithms have been proposed such as CMA-ES \citep{hansen2001completely}, natural evolution strategies \citep{wierstra2014natural}, rank-based genetic algorithm \citep{whitley1989genitor}.

Rank-based ZO algorithms are now extensively used in fields such as evolutionary computation \citep{rechenberg1978evolutionsstrategien,  beyer2001theory}, machine learning and reinforcement learning via CMA-ES \citep{loshchilov2016cma, igel2003neuroevolution,chrabaszcz2018back}, natural evolution strategies \citep{wierstra2014natural, such2017deep},  and engineering and scientific design optimization \citep{hasenjager2005three, back2023evolutionary}. Their success across such diverse domains underscores the practical power of rank-based mechanisms.

Despite their popularity, the theoretical foundations for rank-based ZO algorithms remain incomplete. 
The rank-based selection which is the corner stone of rank-based ZO algorithms brings the nonlinear and discontinuous dependence on order statistics, which complicates standard analytical tools.  
To conquer the difficulties of rank-based selection, drift analysis \citep{he2001drift} and information-geometric optimization \citep{ollivier2017information,akimoto2012theoretical} have become the two main tools in convergence analysis of  rank-based ZO algorithms. 
However, these studies on evolution strategies provide asymptotic insights into stability and adaptive behavior \citep{beyer2001theory, ollivier2017information,wierstra2014natural}, but they do \emph{not} yield explicit non-asymptotic convergence guarantees. 

 In contrast, comparison-based and value-based ZO methods now enjoy explicit complexity bounds under different assumptions \citep{morinaga2021convergence,nesterov2017random, ghadimi2013stochastic,ye2025unified}. 
For example, \citet{nesterov2017random} show that a value-based ZO algorithm can achieve a $\cO\left(\frac{dL}{\mu}\log\frac{1}{\varepsilon}\right)$ query complexity if the function $L$-smooth and $\mu$-strongly convex.
Recently, several theoretical works on evolution strategies show that  evolution strategies algorithms can achieve a linear convergence under proper assumptions \citep{morinaga2021convergence,akimoto2022global}.  
However, these works does \emph{not} provided the convergence rate of rank-based zeroth-order algorithms which common select top-$k$ directions to update the next point. 
Instead, these works only provide the convergence rates of comparison-based algorithms which compare two candidate points.
Thus, despite decades of empirical success and extensive studies of rank-based ZO algorithms, a precise convergence and query analysis still remain an \emph{open} challenge.

The present work aims to contribute toward bridging this theoretical gap by analyzing the structural obstacles in rank-based ZO optimization and by developing tools to characterize their convergence behavior. 
We summary the main contribution as follows:
\begin{itemize}
	\item This paper proposes a simple rank-based zeroth-order algorithm and provides the queries complexities $\tilde{\cO}\left( \frac{dL}{\mu}\log\frac{dL}{\mu\delta}\log\frac{1}{\varepsilon} \right)$ for $L$-smooth and $\mu$-strongly convex function, and $\cO\left(\frac{dL}{\varepsilon}\log\frac{1}{\varepsilon}\right)$ to find an $\varepsilon$-suboptimal solution which hold  with a probability at least $1-\delta$. 
	To my best knowledge, this paper provides the \emph{first}, \emph{explicit}, and \emph{non-asymptotic} query complexities of a rank-based zeroth-order algorithm.  
	\item The analysis techniques in this paper is novel and does \emph{not} lie on the drift analysis \citep{he2001drift} and information-geometric optimization \citep{ollivier2017information} which are popular in the convergence analysis of evolution strategies algorithms. 
	Thus, we believe the analysis technique has potential in the analysis of other rank-based algorithms and will promote the theoretical studies of evolution algorithms. 
	\item Our analysis also provide novel insights over the rank-based ZO algorithms. Our convergence analysis theoretically shows that rank-based ZO algorithm can almost achieve the same query complexity to the one of value-based ZO algorithms. Our convergence analysis also provides new perspective on ``log weights'' which is popular in the famous CMA-ES.
\end{itemize}

\section{Notation and Preliminaries}

First, we will introduce the notation of $\norm{\cdot}$ which is defined as $\norm{\x} \triangleq \sqrt{\sum_{i=1}^{d} x_i^2}$ for a $d$-dimension vector $\x$ with $x_i$ being $i$-th entry of $\x$. 
Letting $\Ub\in\RR^{m\times n}$ be a matrix, the notation $\norm{\Ub}$ is the spectral norm of $\Ub$.
Given two vectors $\x,\y \in \RR^d$, we use $\dotprod{\x,\y} = \sum_{i=1}^{d}x_iy_i$ to denote the inner product of $\x$ and $\y$. 

Next, we will introduce the objective function which we aim to solve in this paper
\begin{equation}\label{eq:det}
	\min_{\x \in \RR^d} f(\x)
\end{equation}
where $f(\x)$ is a  $L$-smooth function. 
We will also introduce two widely used assumptions about the objective function. 

\begin{assumption}\label{ass:L}
A function $f(\x)$ is called $L$-smooth with $L>0$ if for any $\x,\y\in\RR^d$, it holds that
\begin{equation}\label{eq:dyx}
	|d(\y, \x)| \leq \frac{L}{2} \norm{\y-\x}^2, \mbox{ with } d(\y,\x) \triangleq f(\y) - f(\x) - \dotprod{\nabla f(\x), \y -\x}.
\end{equation}
\end{assumption}

\begin{assumption}
A function $f(\x)$ is called $\mu$-strongly convex with $\mu>0$ if for any $\x,\y\in\RR^d$, it holds that
\begin{equation}
	d(\y, \x) \geq \frac{\mu}{2} \norm{\y - \x}^2.
\end{equation}
\end{assumption}
The $L$-smooth and $\mu$-strongly convex assumptions are standard in the convergence analysis of optimization \citep{nesterov2013introductory}.

\section{Algorithm Description}

\begin{algorithm}[t]
	\caption{Rank Based Zeroth-order Algorithm for Smooth Function}
	\label{alg:SA}
	\begin{small}
		\begin{algorithmic}[1]
			\STATE {\bf Input:}
			Initial vector $x_0$, smooth parameter $\alpha > 0$,  sample size $N$, and step size $\eta_t$, the rank oracle $\mathrm{Rank\_Oracle}(\cdot)$.
			\FOR {$t=0,1,2,\dots, T-1$ }
			\STATE Generate $N$ random Gaussian vectors $\ub_i$ with $i = 1,\dots, N$. 
			\STATE Access to the rank oracle $\mathrm{Rank\_Oracle}(\x_t + \alpha \ub_1,\dots, \x_t + \alpha \ub_N)$ and obtain an index $(1),\dots,(N)$ satisfying
			\begin{equation}
				f(\x_t + \alpha \ub_{(1)}) \leq f(\x_t + \alpha \ub_{(2)})\leq \dots  \leq \dots f(\x_t + \alpha \ub_{(N)}).
			\end{equation}
			\STATE Set $w_{(k)}^+ > 0$ with $k = 1,\dots, \frac{N}{4}$ and $w_{(k)}^- < 0 $ with $k = \frac{3N}{4}+1\dots, N$ such that $\sum w_{(k)}^+ = 1$ and $\sum w_{(k)}^- = -1$. For notation convenience, we also use $w_{(k)}$ which omit the superscript $^+,^-$ to denote the weight.
			\STATE Update as 
			\begin{equation}
				\x_{t+1} = \x_t + \eta_t  \sum w_{(k)}  \ub_{(k)}   
			\end{equation}
			\ENDFOR
			\STATE {\bf Output:} $\x_T$
		\end{algorithmic}
	\end{small}
\end{algorithm}

We generate $N$ random Gaussian vectors $\ub_i$ with $i = 1,\dots, N$, that is, $\ub_i \sim \cN(0, \bm{I}_d)$. 
Accordingly, we can get $N$ points $\x_t + \alpha \ub_i$'s. 
Then, we send these points to the rank oracle, that is, $\mathrm{Rank\_Oracle}(\x_t + \alpha \ub_1,\dots, \x_t + \alpha \ub_N)$.
Then we obtain an index $(1),\dots, (N)$ such that
\begin{equation}
	f(\x_t + \alpha \ub_{(1)}) \leq f(\x_t + \alpha \ub_{(2)})\leq \dots  \leq \dots f(\x_t + \alpha \ub_{(N)}).
\end{equation}

For $1\le k \leq N/4$, these points $\x_t + \alpha \ub_{(k)}$ achieve the $\frac{N}{4}$ smallest values. 
We believe that these $\ub_{(k)}$ are probably descent directions.  
Thus, we  give positive weights on these directions. 
Accordingly, we set $w_{(k)}^+ > 0$ with $k = 1,\dots, N/4$ and $\sum w_{(k)}^+ = 1$.

On the other hand, for $ N/4 + 1\le k \leq N$, 
these points $\x_t + \alpha \ub_{(k)}$ achieve the $\frac{N}{4}$ largest values.
We believe that these $\ub_{(k)}$ are probably ascent directions.  
Thus, we want to give negative weights on these directions. 
Accordingly, we set $w_{(k)}^- < 0 $ with $k = 3N/4+1\dots, N$ and $\sum w_{(k)}^- = -1$.
Our algorithm exploits the information in these ``negative'' samples. 
In contrast, there exist algorithms abandon ``negative'' samples. 
For example, the CMA-ES will not use  ``negative'' samples to update the $\x_t$.
However, we will show that Algorithm~\ref{alg:SA} can achieve almost a twice faster convergence rate compared with the counterpart without the information in ``negative'' samples. 
Section~\ref{subsec:discuss} provides a detailed discussion.

In practice, we want to set a larger weight on $\ub_{k_1}$ over $\ub_{k_2}$. That is, set $w_{k_1}^+ > w_{k_2}^+ > 0$,  if $ (k_1) < (k_2)\leq \frac{N}{4}$. 
Similarly, if $ (k_1') > (k_2')\geq \frac{3N}{4}+1$, it commonly sets that $w_{(k_1')}^- < w_{(k_2')}^- < 0$.
A popular weight strategy is the ``log weights'' which is adopted by CMA-ES \citep{hansen2001completely} and natural evolution strategies \citep{wierstra2014natural}. 
However, the average weight, that is setting $w_{(k)} = \frac{4}{N}$, can also achieve a query complexity almost the same to the one with the optimal weight strategy. 
Please refer to  Section~\ref{subsec:discuss} to obtain the detailed derivation and comparison for different weight strategies.

After obtaining these weights $w_{(k)}$'s, we can obtain the descent direction 
\begin{equation}\label{eq:dt}
	\db_t = \sum w_{(k)}  \ub_{(k)}  = \sum_{k=1}^{N/4} w_{(k)}^+ \ub_{(k)} + \sum_{k=3N/4+1}^{N} w_{(k)}^- \ub_{(k)},
\end{equation}
where we use $w_{(k)}$ which omit the superscript $\{^+,^-\}$ to denote the weight for notation convenience.

Finally, we will update $\x_t$ along the direction $\db_t$ with a step size $\eta_t$ and obtain $\x_{t+1}$ as follows:
\begin{equation}
	\x_{t+1} = \x_t + \eta_t \db_t.
\end{equation}  
The detailed algorithm procedure is listed in Algorithm~\ref{alg:SA}.

\section{Convergence Analysis and Query Complexities}

This section will give the detailed convergence analysis and obtain the queries complexities of Algorithm~\ref{alg:SA}.
First, we will give a list of events that hold with high probabilities. 
The high probabilities are determined by the randomness of the algorithm, concentration inequality of the Gaussian distribution, order-statistics of  the Gaussian distribution.
Condition on the events, we provide how the function value will decay when an update is conducted in Algorithm~\ref{alg:SA}. 
Then, we will provide the detailed convergence properties of Algorithm~\ref{alg:SA} and explicit query queries.
Finally, based on the convergence and query results, we discuss the relationship between rank-based and valued-based ZO algorithm, how to choose weights properly, and how the ``negative'' samples help to achieve faster convergence rate.

\subsection{Events}
Assume the objective function $f(\x)$ is $L$-smooth. 
Given $0<\delta<1$, we will define several events as follows:
\begin{align}
	&\cE_{t,1} := \{ |d(\x_t + \alpha \ub_{(k)},\;\x_t)| \leq \Cd L\alpha^2 \mid k = 1,\dots, \frac{N}{4}, \frac{3N}{4}+1, \dots, N \} \label{eq:E1}\\
	&\cE_{t,2} := \{ \norm{\Ub_t}^2 \leq \Ci \mid \Ub_t = [\ub_{(1)}, \dots, \ub_{(\frac{N}{4})}, \ub_{(\frac{3N}{4}+1)}, \ub_{(N)} ]  \} \label{eq:E2}\\
	&\cE_{t,3} := \{\min_{k\in\mathcal{K}} \left| \dotprod{\nabla f(\x_t), \ub_{(k)}}^{-1} \right|
	\geq 
	\left( \sqrt{2\log\frac{2N}{\delta}} \cdot \norm{\nabla f(\x_t)} \right)^{-1}\} \label{eq:E3}\\
	&\cE_{t,4} := \{\dotprod{\nabla f(\x_t), \ub_{(k)}} \geq \norm{\nabla f(\x_t)}, \mid k = \frac{3N}{4}+1,\dots, N\} \label{eq:E4}\\
	&\cE_{t,5} := \{ \dotprod{\nabla f(\x_t), \ub_{(k)}} \leq -\norm{\nabla f(\x_t)}, \mid k = 1, \dots, \frac{N}{4} \}, \label{eq:E5}
\end{align} 
where $\Cd$, and $\Ci$ are defined in Eq.~\eqref{eq:d_up}, and Eq.~\eqref{eq:Ci}, respectively.
The set $\cK$ is defined as $\cK := \{1,\dots, N/4, 3N/4 + 1,\dots, N\}$ and $d(\y, \x)$ is defined in Eq.~\eqref{eq:dyx}.

Next, we will bound the probabilities of above events happen.

\begin{lemma}\label{lem:E1}
Assume that function $f(\x)$ is $L$-smooth.
Given $0<\delta<\frac{2}{N}$, we have
	\begin{equation}\label{eq:d_up}
		\Pr(\cE_{t,1}) \geq 1 - \frac{N\delta}{2}, \mbox{ with } \Cd \triangleq d+2\log\frac{1}{\delta}.
	\end{equation}
\end{lemma}
\begin{proof}
	By the definition of $d(\y, \x)$ in Eq.~\eqref{eq:dyx} with $\y = \x + \alpha \ub$, we can obtain that
\begin{equation}\label{eq:taylor}
	f(\x + \alpha \ub) - f(\x) = \alpha \dotprod{\nabla f(\x), \ub} + d(\x + \alpha \ub, \x), 
\end{equation}
By the $L$-smoothness, we can obtain that
\begin{equation*}
	|d(\x + \alpha \ub, \x)|
	\leq 
	\frac{L \alpha^2}{2} \norm{\ub}^2 \stackrel{\eqref{eq:u_norm}}{\leq} \frac{(2d + 3\log\frac{1}{\delta}) L\alpha^2}{2} \leq (d+2\log\frac{1}{\delta}) L \alpha^2 = \Cd L \alpha^2.
\end{equation*}
	
Accordingly, we can obtain that $|d(\x_t+\alpha \ub, \x_t)| \leq \Cd L\alpha^2$ holds for a random vector $\ub$ with a probability at least $1-\delta$. 
In event $\cE_{t,1}$, there are $N/2$ independent above events. 
By the probability union bound, we can obtain the result.    
\end{proof}

\begin{lemma}
Given $0<\delta<1$, and $\Ci$ defined as follows, we have
\begin{equation}\label{eq:Ci}
	\Pr(\cE_{t,2}) \geq 1 - \delta, \mbox{ with } \Ci := \left( \sqrt{N/2} + \sqrt{d} + \sqrt{2\log\frac{2}{\delta}} \right)^2.
\end{equation}
\end{lemma}
\begin{proof}
By the definition of $\Ub_t$ in $\cE_{t,2}$, $\Ub_t$ is a $d\times \frac{N}{2}$ random Gaussian matrix. 
Then, by Lemma~\ref{lem:U_up}, with a probability at least $1 - \delta$, it holds that
\begin{align*}
	\norm{\Ub_t} \leq \sqrt{N/2} + \sqrt{d} + \sqrt{2\log\frac{2}{\delta}},
\end{align*}
which implies that
\begin{align*}
	\norm{\Ub_t}^2 
	\leq 
	\left( \sqrt{N/2} + \sqrt{d} + \sqrt{2\log\frac{2}{\delta}} \right)^2.
\end{align*}
\end{proof}

\begin{lemma}
	Given $0 < \delta<1$, then event $\cE_{t,3}$ defined in Eq.~\eqref{eq:E3} holds with a probability at least $1 - \delta$.
\end{lemma}
\begin{proof}
	First, we have 
	\begin{equation*}
		\min_{k\in\mathcal{K}} \left| \dotprod{\nabla f(\x_t), \ub_{(k)}}^{-1} \right|
		= \frac{1}{\max_{k\in\mathcal{K}} \left| \dotprod{\nabla f(\x_t), \ub_{(k)}} \right|}
		= \frac{1}{\max_{i=1\dots, N} \left| \dotprod{\nabla f(\x_t), \ub_i} \right|}.
	\end{equation*}
	Note that $\dotprod{\nabla f(\x_t), \ub_i} \sim  \norm{\nabla f(\x_t)} \cdot  \cN(0, 1)$. 
	By Theorem~\ref{thm:gauss_up}, we can obtain that, with a probability at least $1-\delta$, it holds 
	\begin{equation*}
		\max_{i=1\dots, N} \left| \dotprod{\nabla f(\x_t), \ub_i} \right| \leq \sqrt{2\log\frac{2N}{\delta}} \cdot \norm{\nabla f(\x_t)}.
	\end{equation*}
\end{proof}

\begin{lemma}\label{lem:E4}
If event $\cE_{t,1}$ holds	and the smooth parameter $\alpha$ satisfies $\alpha \leq \frac{\norm{\nabla f(\x_t)}}{4L \Cd}$, then we have
\begin{equation}
	\Pr(\cE_{t,4}) \geq 1 - \exp\big(-N\,D(1/4\Vert p)\big) \mbox{ with } p = 0.0224,
\end{equation}
where the KL divergence  $D(q \Vert p) = q \log\frac{q}{p} + (1-q)\log\frac{1-q}{1-p}$ with $0<p,q<1$.
\end{lemma}
\begin{proof}
	In this proof, 	for notation clarity, we use the index $\{j_k\}$ with $k = 1,\dots, N$ instead of $(k)$ with $k=1,2,\dots, N$. 
	Accordingly,  $\{j_k\}$ with $k = 1,\dots, N$ is a permutation of $1,\dots, N$ satisfies that
	\begin{equation}\label{eq:sort0}
		f(\x_t + \alpha \ub_{j_1}) \leq f(\x_t + \alpha \ub_{j_2})\leq \dots \leq f(\x_t + \alpha \ub_{j_N}).
	\end{equation} 
	Then, equivalently, we have
	\begin{align*}
		\frac{f(\x_t + \alpha \ub_{j_1}) - f(\x_t)}{\alpha}	 \leq \frac{f(\x_t + \alpha \ub_{j_2}) - f(\x_t)}{\alpha}\leq \dots \leq \frac{f(\x_t + \alpha \ub_{j_N}) - f(\x_t)}{\alpha}.
	\end{align*}
	Combing with Eq.~\eqref{eq:taylor}, we can conclude that above equation is equivalent to 
	\begin{align*}
		&\dotprod{\nabla f(\x_t), \ub_{j_1}} + \alpha^{-1} d(\x_t + \alpha \ub_{j_1}, \ub_{j_1} )  \leq \dotprod{\nabla f(\x_t), \ub_{j_2}} + \alpha^{-1} d(\x_t + \alpha \ub_{j_2}, \ub_{j_2} )\\
		&\leq \dots \leq \dotprod{\nabla f(\x_t), \ub_{j_N}} + \alpha^{-1} d(\x_t + \alpha \ub_{j_N}, \ub_{j_N} ).
	\end{align*}
	Given the $\ub_1,\dots, \ub_N$, let ${i_k}$ with $k = 1,\dots, N$ be a permutation of $1,\dots, N$ satisfying 
	\begin{equation}\label{eq:sort}
		\dotprod{\nabla f(\x_t), \ub_{i_1}} \leq \dotprod{\nabla f(\x_t), \ub_{i_2}}\leq \dots \leq \dotprod{\nabla f(\x_t), \ub_{i_N}}.
	\end{equation}
	
	Since $\dotprod{\nabla f(\x_t), \ub_i} \sim \norm{\nabla f(\x_t)} \cdot \cN(0,1)$, by Lemma~\ref{lem:low1}, with a probability at least $1-\exp\big(-N\,D(1/4\Vert p)\big)$, it holds that
	\begin{equation}\label{eq:low}
		\dotprod{\nabla f(\x_t), \ub_{i_{3N/4 +1}}} \geq 2 \norm{\nabla f(\x_t)}.
	\end{equation}
	
	Next, we will prove that with a probability at least $1-\exp\big(-N\,D(1/4\Vert p)\big)$, it holds that
	\begin{align*}
		\dotprod{\nabla f(\x_t), \ub_{j_{3N/4 +1}}} \geq \frac{3}{2}\norm{\nabla f(\x_t)}.
	\end{align*}

	If $j_{3N/4 +1} = i_{3N/4 +k}$ with $k\geq 1$, then, we can obtain that
	\begin{align*}
		\dotprod{\nabla f(\x_t), \ub_{j_{3N/4 +1}}} =  \dotprod{\nabla f(\x_t), \ub_{i_{3N/4 +k}}} \geq  \dotprod{\nabla f(\x_t), \ub_{i_{3N/4 +1}}},
	\end{align*} 
	where the last inequality is because of $3N/4 +k \geq 3N/4 +1$ and Eq.~\eqref{eq:sort}.
	Combining with Eq.~\eqref{eq:low}, we can obtain that 
	\begin{equation}
		\dotprod{\nabla f(\x_t), \ub_{j_{3N/4 +1}}} \geq 2 \norm{\nabla f(\x_t)} \geq \frac{3}{2}\norm{\nabla f(\x_t)}.
	\end{equation}
	
	If $j_{3N/4 +1} = i_{3N/4 - k}$ for some $k\geq 0$, then there is at least one index in  $i_{3N/4 +1}, \dots, i_{N}$ replaced by $i_{3N/4 - k}$.
	Without loss of generality, we assume that $i_{3N/4 +1}$ is replaced by $i_{3N/4 - k}$. 
	Accordingly, we have
	\begin{align*}
		f(\x_t + \alpha \ub_{i_{3N/4 - k}}) \geq f(\x_t + \alpha \ub_{i_{3N/4 + 1}}).
	\end{align*}
	Equivalently, we have
	\begin{align*}
		&\dotprod{\nabla f(\x_t), \ub_{i_{3N/4 - k}}} + \alpha^{-1} d(\x_t + \alpha \ub_{i_{3N/4 - k}}, \x_t) \\
		\geq &
		\dotprod{\nabla f(\x_t), \ub_{i_{3N/4 +1}}} + \alpha^{-1} d(\x_t + \alpha \ub_{i_{3N/4 +1}}, \x_t)\\
		\stackrel{\eqref{eq:low}}{\geq}&
		2 \norm{\nabla f(\x_t)} + \alpha^{-1} d(\x_t + \alpha \ub_{i_{3N/4 +1}}, \x_t).
	\end{align*}
	Accordingly, 
	\begin{align*}
		\dotprod{\nabla f(\x_t), \ub_{i_{3N/4 - k}}} 
		\geq& 
		2 \norm{\nabla f(\x_t)} 
		+ \alpha^{-1} d(\x_t + \alpha \ub_{i_{3N/4 +1}}, \x_t) 
		- \alpha^{-1} d(\x_t + \alpha \ub_{i_{3N/4 - k}}, \x_t)\\
		\stackrel{\eqref{eq:d_up}}{\geq}&
		2 \norm{\nabla f(\x_t)} - 2 \Cd L\alpha 
		\geq
		\frac{3}{2} \norm{\nabla f(\x_t)},
	\end{align*}
	where the last inequality is because $\alpha \leq \frac{\norm{\nabla f(\x_t)}}{4L \Cd}$.
	Since $j_{3N/4 +1} = i_{3N/4 - k}$, we have 
	\begin{align*}
		\dotprod{\nabla f(\x_t), \ub_{j_{3N/4 +1}}} \geq \frac{3}{2}\norm{\nabla f(\x_t)}.
	\end{align*}
	
	Now, we have proved that with a probability at least $1 - \delta$ it holds that 
	\begin{equation}\label{eq:j_low}
		\dotprod{\nabla f(\x_t), \ub_{j_{3N/4 +1}}} \geq \frac{3}{2}\norm{\nabla f(\x_t)}.
	\end{equation}
	
	By Eq.~\eqref{eq:sort0}, we have 
	\begin{align*}
		f(\x_t + \alpha \ub_{j_{3N/4+k}}) \geq f(\x_t + \alpha \ub_{j_{3N/4+1}}), \mbox{ for } k>1.
	\end{align*}
	Equivalently, 
	\begin{align}
		&\dotprod{\nabla f(\x_t), \ub_{j_{3N/4 + k}}} + \alpha^{-1} d(\x_t + \alpha \ub_{j_{3N/4 + k}}, \x_t) \\
		\geq &
		\dotprod{\nabla f(\x_t), \ub_{j_{3N/4 +1}}} + \alpha^{-1} d(\x_t + \alpha \ub_{j_{3N/4 +1}}, \x_t)\\
		\stackrel{\eqref{eq:j_low}}{\geq}&
		\frac{3}{2} \norm{\nabla f(\x_t)} + \alpha^{-1} d(\x_t + \alpha \ub_{j_{3N/4 +1}}, \x_t).
	\end{align}
	
	Combining with $\alpha \leq \frac{\norm{\nabla f(\x_t)}}{4L \Cd}$, then we can obtain that for all $k > 1$, it holds 
	\begin{align*}
		\dotprod{\nabla f(\x_t), \ub_{j_{3N/4 +k}}} 
		\geq& 
		\frac{3}{2}\norm{\nabla f(\x_t)} + \alpha^{-1} d(\x_t + \alpha \ub_{j_{3N/4 +1}}, \x_t) - \alpha^{-1} d(\x_t + \alpha \ub_{j_{3N/4 + k}}, \x_t)\\
		\stackrel{\eqref{eq:d_up}}{\geq}& \frac{3}{2}\norm{\nabla f(\x_t)} -  2 \Cd L\alpha 
		\geq \norm{\nabla f(\x_t)}.
	\end{align*}
	
	Combining with Eq.~\eqref{eq:j_low}, we have proved that
	\begin{equation*}
		\dotprod{\nabla f(\x_t), \ub_{j_{3N/4 +k}}}  \geq \norm{\nabla f(\x_t)}, \mbox{ for } k\geq 1.
	\end{equation*}
	Equivalently, we have proved that
	\begin{align*}
		\dotprod{\nabla f(\x_t), \ub_{(k)}} \geq& \norm{\nabla f(\x_t)}, \mbox{ for } k = \frac{3N}{4}+1, \dots, N,
	\end{align*}
	which concludes the proof.
\end{proof}

\begin{lemma}\label{lem:E5}
	If event $\cE_{t,1}$ holds	and the smooth parameter $\alpha$ satisfies $\alpha \leq \frac{\norm{\nabla f(\x_t)}}{4L \Cd}$, then we have
	\begin{equation}
		\Pr(\cE_{t,5}) \geq 1 - \exp\big(-N\,D(1/4\Vert p)\big) \mbox{ with } p = 0.0224,
	\end{equation}
	where the KL divergence  $D(q \Vert p) = q \log\frac{q}{p} + (1-q)\log\frac{1-q}{1-p}$ with $0<p,q<1$.
\end{lemma}
\begin{proof}
The proof is almost the same to the one of Lemma~\ref{lem:E4} except the usage of Lemma~\ref{lem:low1} replaced by Lemma~\ref{lem:low2}.
\end{proof}

\subsection{Descent Analysis}

First, we will show that $\db_t$ is a ``real'' descent direction by proving $\dotprod{\nabla f(\x_t), \db_t} < 0$ if $\alpha$ is sufficiently close to zero. 

\begin{lemma}\label{lem:nab_d}
Assume that function $f(\x)$ is $L$-smooth.
Given $0<\delta<\frac{2}{N}$, and assuming event $\cE_{t,1}$ defined in Eq.~\eqref{eq:E1} holds. 
Then we have
\begin{equation}\label{eq:nab_d}
	\begin{aligned}
	\dotprod{\nabla f(\x_t), \db_t}
	\leq& 
	- \frac{\alpha}{2} \sum_{k\in \mathcal{K}} \frac{w_{(k)}}{f(\x_t) - f(\x_t + \alpha \ub_{(k)})} \cdot \dotprod{\nabla f(\x_t), \ub_{(k)}}^2 \\
	&+ \frac{L^2 \Cd^2\alpha^3}{2}\sum_{k\in \mathcal{K}} \left|\frac{w_{(k)} }{f(\x_t) - f(\x_t + \alpha \ub_{(k)})} \right|. 
	\end{aligned}
\end{equation}
where $\mathcal{K} = \{ 1,\dots, N/4, \;3N/4,\dots, N \}$,
\end{lemma}
\begin{proof}
	By the definition of descent direction $\db_t$, we have
\begin{align*}
\dotprod{\nabla f(\x_t), \db_t}
\stackrel{\eqref{eq:dt}}{=}& 
\sum w_{(k)}\dotprod{\nabla f(\x_t), \ub_{(k)}}
= 
\sum_{k=1}^{N/4} w_{(k)}^+ \dotprod{\nabla f(\x_t), \ub_{(k)}} + \sum_{k=3N/4+1}^{N} w_{(k)}^- \dotprod{\nabla f(\x_t),  \ub_{(k)}}\\
=&
\sum_{k=1}^{N/4} \frac{w_{(k)}^+}{f(\x_t + \alpha \ub_{(k)})  - f(\x_t)} \cdot \Big(f(\x_t + \alpha \ub_{(k)})  - f(\x_t) \Big)  \dotprod{\nabla f(\x_t), \ub_{(k)}}\\
&
+ \sum_{k=3N/4+1}^{N} \frac{w_{(k)}^-}{f(\x_t + \alpha \ub_{(k)})  - f(\x_t)} \cdot \Big(f(\x_t + \alpha \ub_{(k)})  - f(\x_t) \Big)  \dotprod{\nabla f(\x_t), \ub_{(k)}}.
\end{align*}

For $k = 1, \dots, N/4$, we have
\begin{align*}
& \frac{w_{(k)}^+}{f(\x_t + \alpha \ub_{(k)})  - f(\x_t)} \cdot \Big(f(\x_t + \alpha \ub_{(k)})  - f(\x_t) \Big)  \dotprod{\nabla f(\x_t), \ub_{(k)}}\\ 
\stackrel{\eqref{eq:taylor}}{=}& 
\frac{w_{(k)}^+}{f(\x_t + \alpha \ub_{(k)})  - f(\x_t)} \cdot \Big(\alpha \dotprod{\nabla f(\x_t), \ub_{(k)}} +  d(\x_t + \alpha \ub_{(k)}, \x_t) \Big)  \dotprod{\nabla f(\x_t), \ub_{(k)}}\\
=&
 \frac{\alpha w_{(k)}^+}{f(\x_t + \alpha \ub_{(k)})  - f(\x_t)} \dotprod{\nabla f(\x_t), \ub_{(k)}}^2 
+   \frac{w_{(k)}^+ \cdot d(\x_t + \alpha \ub_{(k)}, \x_t)}{f(\x_t + \alpha \ub_{(k)})  - f(\x_t)} \dotprod{\nabla f(\x_t), \ub_{(k)}}\\
\leq&
 \frac{\alpha w_{(k)}^+}{f(\x_t + \alpha \ub_{(k)})  - f(\x_t)} \dotprod{\nabla f(\x_t), \ub_{(k)}}^2 
+
\frac{\alpha w_{(k)}^+}{2 (f(\x_t) - f(\x_t + \alpha \ub_{(k)}))} \dotprod{\nabla f(\x_t), \ub_{(k)}}^2 \\
&
+ \frac{ w_{(k)}^+}{ 2\alpha(f(\x_t) - f(\x_t + \alpha \ub_{(k)}))} \cdot d^2(\x_t + \alpha \ub_{(k)}, \x_t)\\
\leq& 
\frac{\alpha w_{(k)}^+}{2 (f(\x_t + \alpha \ub_{(k)})  - f(\x_t))} \dotprod{\nabla f(\x_t), \ub_{(k)}}^2 
+  \frac{ L^2 \Cd^2 \alpha^3 w_{(k)}^+}{ 2(f(\x_t) - f(\x_t + \alpha \ub_{(k)}))} \\
=&
-  \frac{\alpha w_{(k)}^+}{2 (f(\x_t) - f(\x_t + \alpha \ub_{(k)}))} \dotprod{\nabla f(\x_t), \ub_{(k)}}^2 
+   \frac{ L^2 \Cd^2 \alpha^3 w_{(k)}^+}{ 2(f(\x_t) - f(\x_t + \alpha \ub_{(k)}))},
\end{align*}
where the first inequality is because of fact that $ab\leq \frac{a^2 + b^2}{2}$ for all real numbers $a, b$ and the last inequality is because  event $\cE_{t,1}$ holds.

Similarly,
\begin{align*}
& \frac{w_{(k)}^-}{f(\x_t + \alpha \ub_{(k)})  - f(\x_t)} \cdot \Big(f(\x_t + \alpha \ub_{(k)})  - f(\x_t) \Big)  \dotprod{\nabla f(\x_t), \ub_{(k)}}\\
=&
 \frac{\alpha w_{(k)}^-}{f(\x_t + \alpha \ub_{(k)})  - f(\x_t)} \cdot \dotprod{\nabla f(\x_t), \ub_{(k)}}^2 
+ d(\x_t + \alpha \ub_{(k)}, \x_t) \cdot    \dotprod{\nabla f(\x_t), \ub_{(k)}}\\
\leq&
- \frac{\alpha w_{(k)}^-}{2 (f(\x_t) - f(\x_t + \alpha \ub_{(k)})) } \cdot \dotprod{\nabla f(\x_t), \ub_{(k)}}^2
+ \frac{L^2 \Cd^2\alpha^3 w_{(k)}^- }{ 2( f(\x_t) - f(\x_t + \alpha \ub_{(k)})) }
\end{align*}

Combining above results, we can obtain that
\begin{align*}
\dotprod{\nabla f(\x_t), \db_t}
\leq 
- \frac{\alpha}{2} \sum_{k\in \mathcal{K}} \frac{w_{(k)}}{f(\x_t) - f(\x_t + \alpha \ub_{(k)})} \cdot \dotprod{\nabla f(\x_t), \ub_{(k)}}^2 
+ \frac{L^2 \Cd^2\alpha^3}{2}\sum_{k\in \mathcal{K}} \left|\frac{w_{(k)} }{f(\x_t) - f(\x_t + \alpha \ub_{(k)})} \right|. 
\end{align*}
\end{proof}

Now, we will provide the upper bound of $\norm{\db_t}^2$. 
This upper bound plays an important role in the convergence analysis.
\begin{lemma}\label{lem:d_norm}
Given $0<\delta <1$ and assuming event $\cE_{t,2}$ defined in Eq.~\eqref{eq:E2} holds, then it holds that
\begin{equation}\label{eq:d_norm}
	\norm{\db_t}^2 
	\leq 
	C_{N, d,\delta} \sum_{k\in \mathcal{K}} \frac{w_{(k)}^2}{\dotprod{\nabla f(\x_t), \ub_{(k)}}^2}\dotprod{\nabla f(\x_t), \ub_{(k)}}^2.
\end{equation}
\end{lemma}
\begin{proof}
By the definition of $\db_t$ in Eq.~\eqref{eq:dt}, we can obtain  
\begin{align*}
	\norm{\db_t}^2 
	= \norm{ \sum_{k\in \mathcal{K}} w_{(k)} \ub_{(k)} }^2 
	= \norm{\Ub_t \wb}^2
	\leq \norm{\Ub_t}^2 \cdot \norm{\wb}^2
	\leq C_{N, d,\delta} \norm{\wb}^2,
\end{align*}
where $\wb$ is the vector with $w_{(k)}$ being its $k$-th entry and  the last inequality is because event $\cE_{t,2}$ holds.
Furthermore, we have
\begin{align*}
\norm{\wb}^2 
= \sum_{k\in \mathcal{K}} \frac{w_{(k)}^2}{\dotprod{\nabla f(\x_t), \ub_{(k)}}^2}\dotprod{\nabla f(\x_t), \ub_{(k)}}^2.
\end{align*}
Combining above results, we can conclude the proof.
\end{proof}

Based on above two lemmas, we will obtain the first result that describes how the function value decays as iteration goes in Algorithm~\ref{alg:SA} and what a step size should properly be chosen.
\begin{lemma}\label{lem:dec}
Assume that function $f(\x)$ is $L$-smooth.
Given $0<\delta<\frac{2}{N}$,  assume events $\cE_{t,1}$ and $\cE_{t,2}$ defined in Eq.~\eqref{eq:E1}-\eqref{eq:E2} hold. 
By setting step size $\eta_t = \min_{k\in\mathcal{K}} \frac{\dotprod{\nabla f(\x_t), \ub_{(k)}}^2}{2L\Ci w_{(k)}} \cdot \left( \frac{f(\x_t) - f(\x_t + \alpha \ub_{(k)})}{\alpha} \right)^{-1}$, we can obtain that
\begin{equation}\label{eq:dec}
\begin{aligned}
	f(\x_{t+1}) 
	\leq& 
	f(\x_t) 
	- \frac{1}{4} \cdot \min_{k\in\mathcal{K}} \frac{\dotprod{\nabla f(\x_t), \ub_{(k)}}^2}{2L\Ci w_{(k)}} \cdot \left( \frac{f(\x_t) - f(\x_t + \alpha \ub_{(k)})}{\alpha} \right)^{-1} \\
	&\cdot \min_{k\in\mathcal{K}} \frac{w_{(k)}}{(f(\x_t) - f(\x_t + \alpha \ub_{(k)}))/\alpha} \cdot  \sum_{k\in \mathcal{K}} \dotprod{\nabla f(\x_t), \ub_{(k)}}^2 
	+ \Delta_{\alpha, 1},
\end{aligned}
\end{equation}
where we denote $\Delta_{\alpha,1}$ 
\begin{equation}\label{eq:Delta_1}
\Delta_{\alpha, 1} := \frac{\eta_t L^2 \Cd^2\alpha^3}{2}\sum_{k\in \mathcal{K}} \left|\frac{w_{(k)} }{f(\x_t) - f(\x_t + \alpha \ub_{(k)})} \right|.
\end{equation}
\end{lemma}
\begin{proof}
	By the $L$-smoothness of $f(\x)$, we have
\begin{align*}
f(\x_{t+1}) 
\stackrel{\eqref{eq:dyx}}{\leq}& 
f(\x_t) + \eta_t \dotprod{\nabla f(\x_t), \db_t} + \frac{L\eta_t^2}{2} \norm{\db_t}^2\\
\stackrel{\eqref{eq:nab_d}\eqref{eq:d_norm}}{\leq}&
f(\x_t)
- \frac{ \eta_t}{2} \sum_{k\in \mathcal{K}} \frac{w_{(k)}}{(f(\x_t) - f(\x_t + \alpha \ub_{(k)}))/\alpha} \cdot \dotprod{\nabla f(\x_t), \ub_{(k)}}^2 \\
&
+ \frac{\eta_t L^2 \Cd^2\alpha^3}{2}\sum_{k\in \mathcal{K}} \left|\frac{w_{(k)} }{f(\x_t) - f(\x_t + \alpha \ub_{(k)})} \right|\\
& + \frac{\eta_t^2 L C_{N, d,\delta}}{2} \sum_{k\in \mathcal{K}} \frac{w_{(k)}^2}{\dotprod{\nabla f(\x_t), \ub_{(k)}}^2}\dotprod{\nabla f(\x_t), \ub_{(k)}}^2\\
=&
f(\x_t) 
- \frac{  \eta_t}{2} \sum_{k\in \mathcal{K}} \frac{w_{(k)}}{(f(\x_t) - f(\x_t + \alpha \ub_{(k)}))/\alpha}\\
&\cdot
\left( 1 -  \frac{\eta_t L C_{N, d,\delta} w_{(k)}}{\dotprod{\nabla f(\x_t), \ub_{(k)}}^2} \cdot \frac{f(\x_t) - f(\x_t + \alpha \ub_{(k)})}{\alpha} \right)\dotprod{\nabla f(\x_t), \ub_{(k)}}^2\\
& 
+ \frac{ \eta_t L^2 \Cd^2\alpha^3}{2}\sum_{k\in \mathcal{K}} \left|\frac{w_{(k)} }{f(\x_t) - f(\x_t + \alpha \ub_{(k)})} \right|\\
\leq&
f(\x_t) 
- \frac{  \eta_t}{4} \sum_{k\in \mathcal{K}} \frac{w_{(k)}}{(f(\x_t) - f(\x_t + \alpha \ub_{(k)}))/\alpha}\dotprod{\nabla f(\x_t), \ub_{(k)}}^2\\
& 
+ \frac{\eta_t L^2 \Cd^2\alpha^3}{2}\sum_{k\in \mathcal{K}} \left|\frac{w_{(k)} }{f(\x_t) - f(\x_t + \alpha \ub_{(k)})} \right|,
\end{align*}
where the last inequality is because of $\eta_t = \min_{k\in\mathcal{K}} \frac{\dotprod{\nabla f(\x_t), \ub_{(k)}}^2}{2L\Ci w_{(k)}} \cdot \left( \frac{f(\x_t) - f(\x_t + \alpha \ub_{(k)})}{\alpha} \right)^{-1}$.

Furthermore,
\begin{align*}
&- \sum_{k\in \mathcal{K}} \frac{w_{(k)}}{(f(\x_t) - f(\x_t + \alpha \ub_{(k)}))/\alpha}\dotprod{\nabla f(\x_t), \ub_{(k)}}^2 \\
\leq& 
- \min_{k\in\mathcal{K}} \frac{w_{(k)}}{(f(\x_t) - f(\x_t + \alpha \ub_{(k)}))/\alpha} \cdot  \sum_{k\in \mathcal{K}} \dotprod{\nabla f(\x_t), \ub_{(k)}}^2 
\end{align*}

Therefore, we can obtain that
\begin{align*}
f(\x_{t+1}) 
\leq& 
f(\x_t) 
- \frac{1}{4} \cdot \min_{k\in\mathcal{K}} \frac{\dotprod{\nabla f(\x_t), \ub_{(k)}}^2}{2L\Ci w_{(k)}} \cdot \left( \frac{f(\x_t) - f(\x_t + \alpha \ub_{(k)})}{\alpha} \right)^{-1} \\
&\cdot \min_{k\in\mathcal{K}} \frac{w_{(k)}}{(f(\x_t) - f(\x_t + \alpha \ub_{(k)}))/\alpha} \cdot  \sum_{k\in \mathcal{K}} \dotprod{\nabla f(\x_t), \ub_{(k)}}^2\\
& 
+ \frac{\eta_t L^2 \Cd^2\alpha^3}{2}\sum_{k\in \mathcal{K}} \left|\frac{w_{(k)} }{f(\x_t) - f(\x_t + \alpha \ub_{(k)})} \right|.
\end{align*}
\end{proof}

Next,  to guarantee that Eq.~\eqref{eq:dec} can achieve a sufficient large value decay, we will lower bound the values of  $\min_{k\in\mathcal{K}} \frac{\dotprod{\nabla f(\x_t), \ub_{(k)}}^2}{2L\Ci w_{(k)}} \cdot \left( \frac{f(\x_t) - f(\x_t + \alpha \ub_{(k)})}{\alpha} \right)^{-1}$ and $\min_{k\in\mathcal{K}} \frac{w_{(k)}}{(f(\x_t) - f(\x_t + \alpha \ub_{(k)}))/\alpha}$. 
\begin{lemma}\label{lem:low_up}
Assume that events $\cE_{t,1}$, $\cE_{t,4}$ and $\cE_{t,5}$ hold.  
If the smooth parameter $\alpha$ is sufficient small such that  $\alpha \leq \frac{\norm{\nabla f(\x_t)}}{L \Cd }$ holds for all $\x_t$ and $\ub_{(k)}$, then it holds that
\begin{align}
\frac{\dotprod{\nabla f(\x_t), \ub_{(k)}}^2}{ w_{(k)}} \cdot \left( \frac{f(\x_t) - f(\x_t + \alpha \ub_{(k)})}{\alpha} \right)^{-1}
\geq& 
\left|\frac{\dotprod{\nabla f(\x_t), \ub_{(k)}}}{2 w_{(k)}}\right| \label{eq:low1}\\
\frac{w_{(k)}}{(f(\x_t) - f(\x_t + \alpha \ub_{(k)}))/\alpha} 
\geq&  
\frac{1}{2} |w_{(k)} \dotprod{\nabla f(\x_t), \ub_{(k)}}^{-1} |, \label{eq:low2}
\end{align}
and
\begin{align}
	\frac{\dotprod{\nabla f(\x_t), \ub_{(k)}}^2}{ w_{(k)}} \cdot \left( \frac{f(\x_t) - f(\x_t + \alpha \ub_{(k)})}{\alpha} \right)^{-1}
	\leq& 
	\left|\frac{2\dotprod{\nabla f(\x_t), \ub_{(k)}}}{ w_{(k)}}\right|, \label{eq:up1}\\
	\left| \frac{w_{(k)}}{(f(\x_t) - f(\x_t + \alpha \ub_{(k)}))/\alpha}\right| 
	\leq&
	\frac{2 |w_{(k)}|}{| \dotprod{\nabla f(\x_t), \ub_{(k)}}|}. \label{eq:up2}
\end{align}
\end{lemma}
\begin{proof}
	We have
\begin{align*}
&\frac{\dotprod{\nabla f(\x_t), \ub_{(k)}}^2}{ w_{(k)}} \cdot \left( \frac{f(\x_t) - f(\x_t + \alpha \ub_{(k)})}{\alpha} \right)^{-1}\\ 
=& 
\left| \frac{\dotprod{\nabla f(\x_t), \ub_{(k)}}^2}{ w_{(k)}} \cdot \left( \frac{f(\x_t) - f(\x_t + \alpha \ub_{(k)})}{\alpha} \right)^{-1}  \right|\\
\stackrel{\eqref{eq:taylor}}{=}&
\left|
\frac{\dotprod{\nabla f(\x_t), \ub_{(k)}}^2}{ w_{(k)}} \cdot \left( \dotprod{\nabla f(\x_t, \ub_{(k)})} - \alpha^{-1} \cdot d(\x_t + \alpha \ub_{(k)}, \x_t) \right)^{-1} \right|\\
\geq&
\left|
\frac{\dotprod{\nabla f(\x_t), \ub_{(k)}}^2}{ w_{(k)}} \cdot \left( |\dotprod{\nabla f(\x_t, \ub_{(k)})}| + \alpha^{-1}|d(\x_t + \alpha \ub_{(k)}, \x_t)| \right)^{-1} \right|\\
\stackrel{\eqref{eq:d_up}}{\geq}&
\left|
\frac{\dotprod{\nabla f(\x_t), \ub_{(k)}}^2}{ w_{(k)}} \cdot \left( |\dotprod{\nabla f(\x_t, \ub_{(k)})}| + \Cd L \alpha \right)^{-1} \right|\\
\geq&
\left|
\frac{\dotprod{\nabla f(\x_t), \ub_{(k)}}^2}{ w_{(k)}} \cdot  | 2\dotprod{\nabla f(\x_t, \ub_{(k)})}|^{-1} \right|\\
=&
\left|\frac{\dotprod{\nabla f(\x_t), \ub_{(k)}}}{2 w_{(k)}}\right|,
\end{align*}
where the second inequality is because event $\cE_{t,1}$ holds, the last inequality is because of $\alpha$ is sufficient small such that $\alpha \leq \frac{|\dotprod{\nabla f(\x_t), \ub_{(k)}}|}{L \Cd }$ implied by events $\cE_{t,4}$, $\cE_{t,5}$ and the condition $\alpha \leq \frac{\norm{\nabla f(\x_t)}}{L \Cd }$. 

Similarly, if $\alpha \leq \frac{\norm{\nabla f(\x_t)}}{L \Cd }$, we can obtain that
\begin{align*}
 \frac{w_{(k)}}{(f(\x_t) - f(\x_t + \alpha \ub_{(k)}))/\alpha}
 \geq& 
 \left| w_{(k)} \left( |\dotprod{\nabla f(\x_t, \ub_{(k)})}| + \Cd L \alpha \right)^{-1} \right|\\
 \geq& 
 \frac{1}{2} |w_{(k)} \dotprod{\nabla f(\x_t), \ub_{(k)}}^{-1} |.
\end{align*}

We also have
\begin{align*}
&\frac{\dotprod{\nabla f(\x_t), \ub_{(k)}}^2}{ w_{(k)}} \cdot \left( \frac{f(\x_t) - f(\x_t + \alpha \ub_{(k)})}{\alpha} \right)^{-1} \\
\leq &
\left|
\frac{\dotprod{\nabla f(\x_t), \ub_{(k)}}^2}{ w_{(k)}} \cdot \left( |\dotprod{\nabla f(\x_t, \ub_{(k)})}| - \Cd L \alpha \right)^{-1} \right| \\
\leq&
\left|\frac{2\dotprod{\nabla f(\x_t), \ub_{(k)}}}{ w_{(k)}}\right|,
\end{align*}
where the last inequality is because of $\alpha \leq \frac{\norm{\nabla f(\x_t)}}{L \Cd }$.

Similarly, we can obtain that
\begin{align*}
&\left| \frac{w_{(k)}}{(f(\x_t) - f(\x_t + \alpha \ub_{(k)}))/\alpha}\right| 
=
\frac{ |w_{(k)}| }{| \dotprod{\nabla f(\x_t, \ub_{(k)})} - \alpha^{-1} \cdot d(\x_t + \alpha \ub_{(k)}, \x_t) |} \\
\leq&
\frac{ |w_{(k)}| }{| \dotprod{\nabla f(\x_t, \ub_{(k)})}| - \Cd L \alpha } 
\leq
\frac{2 |w_{(k)}|}{| \dotprod{\nabla f(\x_t, \ub_{(k)})}|}.
\end{align*}

\end{proof}

Based on Lemma~\ref{lem:dec} and Lemma~\ref{lem:low_up}, we can obtain the following lemma which has a clean describe how function value decay after one update in Algorithm~\ref{alg:SA}.

\begin{lemma}\label{lem:dec1}
Assume that conditions required in Lemma~\ref{lem:dec} and Lemma~\ref{lem:low_up} all hold, then it holds that
\begin{equation}\label{eq:dec1}
\begin{aligned}
	f(\x_{t+1}) 
	\leq& f(\x_t) 
	- \frac{1}{16L\Ci} \cdot\min_{k\in\mathcal{K}} \left|\frac{\dotprod{\nabla f(\x_t), \ub_{(k)}}}{w_{(k)}}\right|
	\cdot \min_{k\in\mathcal{K}} \left|\frac{w_{(k)}}{\dotprod{\nabla f(\x_t), \ub_{(k)}}}\right|\\
	&
	\cdot  \sum_{k\in \mathcal{K}} \dotprod{\nabla f(\x_t), \ub_{(k)}}^2 + \Delta_{\alpha,1}.
\end{aligned}
\end{equation}
\end{lemma}
\begin{proof}
	We have
\begin{align*}
f(\x_{t+1}) 
\stackrel{\eqref{eq:dec}}{\leq}& 
f(\x_t) 
- \frac{1}{4} \cdot \min_{k\in\mathcal{K}} \frac{\dotprod{\nabla f(\x_t), \ub_{(k)}}^2}{2L\Ci w_{(k)}} \cdot \left( \frac{f(\x_t) - f(\x_t + \alpha \ub_{(k)})}{\alpha} \right)^{-1} \\
&\cdot \min_{k\in\mathcal{K}} \frac{w_{(k)}}{(f(\x_t) - f(\x_t + \alpha \ub_{(k)}))/\alpha} \cdot  \sum_{k\in \mathcal{K}} \dotprod{\nabla f(\x_t), \ub_{(k)}}^2 + \Delta_{\alpha+1}\\
\stackrel{\eqref{eq:low1}\eqref{eq:low2}}{\leq}&
f(\x_t) 
- \frac{1}{16L\Ci} \cdot\min_{k\in\mathcal{K}} \left|\frac{\dotprod{\nabla f(\x_t), \ub_{(k)}}}{w_{(k)}}\right|
\cdot \min_{k\in\mathcal{K}} \left|\frac{w_{(k)}}{\dotprod{\nabla f(\x_t), \ub_{(k)}}}\right|\\
&
\cdot  \sum_{k\in \mathcal{K}} \dotprod{\nabla f(\x_t), \ub_{(k)}}^2 + \Delta_{\alpha+1}
\end{align*}	
\end{proof}

\begin{lemma}\label{lem:dec2}
Assume that conditions required in Lemma~\ref{lem:dec1} and events $\cE_{t,4}$ and $\cE_{t,5}$ hold, then it holds that
\begin{equation}\label{eq:dec2}
f(\x_{t+1})
\leq 
f(\x_t) 
- \frac{\min_{k\in\mathcal{K}} \; w_{(k)}}{\max_{k\in\mathcal{K}} \; w_{(k)}} \cdot \frac{\Ck}{16L\Ci\cdot \sqrt{2\log\frac{2N}{\delta}}} \norm{\nabla f(\x_t)}^2 
+ \Delta_{\alpha,1}.
\end{equation}
\end{lemma}
\begin{proof}
Combining Eq.~\eqref{eq:dec1} with the fact that
\begin{align*}
\min_{k\in\mathcal{K}} \left|\frac{\dotprod{\nabla f(\x_t), \ub_{(k)}}}{w_{(k)}}\right|
\cdot \min_{k\in\mathcal{K}} \left|\frac{w_{(k)}}{\dotprod{\nabla f(\x_t), \ub_{(k)}}}\right|
\geq 
\frac{\min_{k\in\mathcal{K}} \; w_{(k)}}{\max_{k\in\mathcal{K}} \; w_{(k)}} 
\cdot \frac{\min_{k\in\mathcal{K}} \; |\dotprod{\nabla f(\x_t), \ub_{(k)}}|}{\max_{k\in\mathcal{K}} \; |\dotprod{\nabla f(\x_t), \ub_{(k)}}|},
\end{align*}
we can obtain that
	\begin{align*}
		f(\x_{t+1})
		\leq& 
		f(\x_t) 
		- \frac{1}{16L\Ci} \cdot \frac{\min_{k\in\mathcal{K}} \; w_{(k)}}{\max_{k\in\mathcal{K}} \; w_{(k)}} 
		\cdot \frac{\min_{k\in\mathcal{K}} \; |\dotprod{\nabla f(\x_t), \ub_{(k)}}|}{\max_{k\in\mathcal{K}} \; |\dotprod{\nabla f(\x_t), \ub_{(k)}}|}  \\
		&\cdot \sum_{k\in \mathcal{K}} \dotprod{\nabla f(\x_t), \ub_{(k)}}^2 + \Delta_{\alpha+1}\\
		\leq&
		f(\x_t) 
		- \frac{1}{16L\Ci} \cdot \frac{\min_{k\in\mathcal{K}} \; w_{(k)}}{\max_{k\in\mathcal{K}} \; w_{(k)}} 
		\cdot \frac{\norm{\nabla f(\x_t)}}{\sqrt{2\log\frac{2N}{\delta}}\norm{\nabla f(\x_t)}}\cdot \sum_{k\in \mathcal{K}} \dotprod{\nabla f(\x_t), \ub_{(k)}}^2 + \Delta_{\alpha+1}\\
		\leq&
		f(\x_t) 
		- \frac{\min_{k\in\mathcal{K}} \; w_{(k)}}{\max_{k\in\mathcal{K}} \; w_{(k)}} \cdot \frac{\Ck}{16L\Ci\cdot \sqrt{2\log\frac{2N}{\delta}}} \cdot \norm{\nabla f(\x_t)}^2 
		+ \Delta_{\alpha,1},
	\end{align*}
	where the second inequality is because events $\cE_{t,3}$,  $\cE_{t,4}$ and $\cE_{t,5}$ hold, the last inequality is because events $\cE_{t,4}$ and $\cE_{t,5}$ hold.
\end{proof}

\begin{lemma}\label{lem:del_up}
Condition on $\cE_{t,3}$, $\cE_{t,4}$, and $\cE_{t,5}$ hold, then $\Delta_{\alpha,1}$ can be upper bounded as
\begin{equation}\label{eq:del_up}
\Delta_{\alpha, 1} 
\leq 
\frac{\max_{k\in\mathcal{K}} \; w_{(k)}}{\min_{k\in\mathcal{K}} \; w_{(k)}} 
\cdot 
\frac{ N L \Cd^2 \cdot \sqrt{2\log\frac{2N}{\delta}} \cdot \alpha^2 }{2\Ci}.
\end{equation}
\end{lemma}
\begin{proof}
	By the definition of $\Delta_{\alpha,1}$, we have
\begin{align*}
\Delta_{\alpha, 1} 
\stackrel{\eqref{eq:Delta_1}}{=}& 
\frac{\eta_t L^2 \Cd^2\alpha^3}{2}\sum_{k\in \mathcal{K}} \left|\frac{w_{(k)} }{f(\x_t) - f(\x_t + \alpha \ub_{(k)})} \right|\\
=& 
\min_{k\in\mathcal{K}} \frac{\dotprod{\nabla f(\x_t), \ub_{(k)}}^2}{2L\Ci w_{(k)}} \cdot \left( \frac{f(\x_t) - f(\x_t + \alpha \ub_{(k)})}{\alpha} \right)^{-1}  
\cdot 
\frac{ L^2 \Cd^2\alpha^2}{2}\sum_{k\in \mathcal{K}} \left|\frac{w_{(k)} }{(f(\x_t) - f(\x_t + \alpha \ub_{(k)}))/\alpha} \right| \\
\stackrel{\eqref{eq:up1}\eqref{eq:up2}}{\leq}&
\min_{k\in\mathcal{K}} 	\left|\frac{\dotprod{\nabla f(\x_t), \ub_{(k)}}}{ L \Ci w_{(k)}}\right| 
\cdot 
 L^2 \Cd^2\alpha^2 \cdot \sum_{k\in \mathcal{K}} \frac{ |w_{(k)}|}{| \dotprod{\nabla f(\x_t, \ub_{(k)})}|}\\
\leq&
\frac{\max_{k\in\mathcal{K}} \; w_{(k)}}{\min_{k\in\mathcal{K}} \; w_{(k)}} 
\cdot
\frac{L^2 \Cd^2\alpha^2}{L\Ci}   
\cdot 
\sqrt{2\log\frac{2N}{\delta}} \cdot \norm{\nabla f(\x_t)} 
\cdot \frac{N}{2} \frac{1}{\norm{\nabla f(\x_t)} }\\
=&
\frac{\max_{k\in\mathcal{K}} \; w_{(k)}}{\min_{k\in\mathcal{K}} \; w_{(k)}} 
\cdot 
\frac{ N L \Cd^2 \cdot \sqrt{2\log\frac{2N}{\delta}} \cdot \alpha^2 }{2\Ci},
\end{align*}
where the second equality is because of  $\eta_t = \min_{k\in\mathcal{K}} \frac{\dotprod{\nabla f(\x_t), \ub_{(k)}}^2}{2L\Ci w_{(k)}} \cdot \left( \frac{f(\x_t) - f(\x_t + \alpha \ub_{(k)})}{\alpha} \right)^{-1}$ shown in Lemma~\ref{lem:dec} and the last inequality is because events $\cE_{t,3}$, $\cE_{t,4}$, and $\cE_{t,5}$ hold.
\end{proof}

\subsection{Main Theorems and Query Complexities}

First, we will give two main theorems which describe the convergence properties of Algorithm~\ref{alg:SA} when the objective function is $L$-smooth and $\mu$-strongly convex, and $L$-smooth but maybe nonconvex, respectively.
\begin{theorem}\label{thm:main}
	Assume that the objective function $f(\x)$ is $L$-smooth and $\mu$-strongly convex. 
	Given $0<\delta<1$, 
	let $\cE_t = \cap_{i=1,\dots, 5}\; \cE_{t,i}$ be the event that events $\cE_{t,1}$-$\cE_{t,5}$ (defined in Eq.~\eqref{eq:E1}-Eq.~\eqref{eq:E5}) all hold.
	Set the smooth parameter $\alpha \leq \frac{\norm{\nabla f(\x_t)}}{L \Cd }$.
	By setting the step size   $\eta_t = \min_{k\in\mathcal{K}} \frac{\dotprod{\nabla f(\x_t), \ub_{(k)}}^2}{2L\Ci w_{(k)}} \cdot \left( \frac{f(\x_t) - f(\x_t + \alpha \ub_{(k)})}{\alpha} \right)^{-1}$, then Algorithm~\ref{alg:SA} has the following convergence property
	\begin{equation}\label{eq:main_lin}
	\begin{aligned}
	f(\x_{t+1}) - f(\x^*) 
	\leq&
	\left(1 - \frac{\min_{k\in\mathcal{K}} \; w_{(k)}}{\max_{k\in\mathcal{K}} \; w_{(k)}} 
	\cdot 
	\frac{N}{16\Ci\cdot \sqrt{2\log\frac{2N}{\delta}}} 
	\cdot  
	\frac{\mu}{L}  
	\right) 
	\cdot 
	\Big( f(\x_t) - f(\x^*) \Big) + \Delta_{\alpha,1}',
	\end{aligned}
	\end{equation}
	with $\Delta_{\alpha,1}'$ defined as
	\begin{equation}\label{eq:del_p}
	\Delta_{\alpha,1}' = \frac{\max_{k\in\mathcal{K}} \; w_{(k)}}{\min_{k\in\mathcal{K}} \; w_{(k)}} 
	\cdot 
	\frac{ N L \Cd^2 \cdot \sqrt{2\log\frac{2N}{\delta}} \cdot \alpha^2 }{2\Ci},
	\end{equation}
	where $\x^*$ the minimal point of $f(\x)$ and  $\Cd$, and $\Ci$  are defined in Eq.~\eqref{eq:d_up},  and  Eq.~\eqref{eq:Ci}, respectively.
	
	Furthermore, 
	\begin{equation}\label{eq:prob}
		\Pr(\cE_t) \geq 1 - \left( \frac{(N+ 4)\delta}{2} + 2 \cdot \exp(- N \cdot D(1/4 \Vert p))\right), \mbox{ with } p = 0.0224,
	\end{equation}
	where the KL divergence  $D(q \Vert p) = q \log\frac{q}{p} + (1-q)\log\frac{1-q}{1-p}$.
\end{theorem}
\begin{proof}
By Lemma~\ref{lem:dec2}, we have
\begin{align*}
& f(\x_{t+1}) - f(\x^*) 
\stackrel{\eqref{eq:dec2}}{\leq}
f(\x_t) - f(\x^*)
- \frac{\min_{k\in\mathcal{K}} \; w_{(k)}}{\max_{k\in\mathcal{K}} \; w_{(k)}} \cdot \frac{\Ck}{16L\Ci\cdot \sqrt{2\log\frac{2N}{\delta}}} \norm{\nabla f(\x_t)}^2 
+ \Delta_{\alpha,1}\\
\leq&
f(\x_t) - f(\x^*)
- 
\frac{\min_{k\in\mathcal{K}} \; w_{(k)}}{\max_{k\in\mathcal{K}} \; w_{(k)}} 
\cdot  
\frac{\mu}{L} 
\cdot 
\frac{\Ck}{8L\Ci\cdot \sqrt{2\log\frac{2N}{\delta}}}\cdot \Big( f(\x_t) - f(\x^*) \Big)
+ 
\Delta_{\alpha,1}\\
=&
\left(1 - \frac{\min_{k\in\mathcal{K}} \; w_{(k)}}{\max_{k\in\mathcal{K}} \; w_{(k)}} 
\cdot 
\frac{\Ck}{8\Ci\cdot \sqrt{2\log\frac{2N}{\delta}}} 
\cdot  
\frac{\mu}{L}  
\right) 
\cdot 
\Big( f(\x_t) - f(\x^*) \Big) + \Delta_{\alpha,1},
\end{align*}
where the second inequality is because of Lemma~\ref{lem:str_cvx}.
Combining with Lemma~\ref{lem:del_up}, we can obtain the result in Eq.~\eqref{eq:main_lin}.

Finally, by a union bound, we have
\begin{align*}
	\Pr(\cE_t) 
	=&
	\Pr\left( \forall_i\; \mbox{event } \cE_{t,i} \mbox{ holds } \right) \\
	\geq& 
	1 - 
	\left( 
	\frac{N\delta}{2} + \delta + \delta + \exp(- N \cdot D(1/4 \Vert p)) + \exp(- N \cdot D(1/4 \Vert p))
	\right)\\
	=&
	1 - \left( \frac{(N+ 4)\delta}{2} + 2 \cdot \exp\Big(- N D(1/4 \Vert p)\Big)\right),
\end{align*}
where the inequality is because of the probabilities of events $\cE_{t,1}$-$\cE_{t,5}$ shown in Lemma~\ref{lem:E1}-Lemma~\ref{lem:E5}.
\end{proof}

\begin{theorem}\label{thm:main1}
	Assume that the objective function $f(\x)$ is $L$-smooth and its optimal value $f(\x^*)>-\infty$. 
	Given $0<\delta<1$ and iteration number $0<T$, 
	let $\cE_t = \cap_{i=1,\dots, 5}\; \cE_{t,i}$ be the event that events $\cE_{t,1}$-$\cE_{t,5}$ (defined in Eq.~\eqref{eq:E1}-Eq.~\eqref{eq:E5}) all hold. 
	Let events $\cE_t$ hold for $t = 0,\dots, T-1$.
	Set the smooth parameter $\alpha \leq \frac{\norm{\nabla f(\x_t)}}{L \Cd }$ for $t = 0,\dots, T-1$.
	By setting the step size   $\eta_t = \min_{k\in\mathcal{K}} \frac{\dotprod{\nabla f(\x_t), \ub_{(k)}}^2}{2L\Ci w_{(k)}} \cdot \left( \frac{f(\x_t) - f(\x_t + \alpha \ub_{(k)})}{\alpha} \right)^{-1}$, then Algorithm~\ref{alg:SA} has the following convergence property
	\begin{equation}\label{eq:main_sub}
		\begin{aligned}
			\frac{1}{T}\sum_{t=0}^{T-1} \norm{\nabla f(\x_t)}^2
			\leq& 
			\frac{f(\x_0) - f(\x^*)}{T}
			\cdot 
			\frac{32L\Ci\cdot \sqrt{2\log\frac{2N}{\delta}}}{N}
			\cdot 
			\frac{\max_{k\in\mathcal{K}} \; w_{(k)}}{\min_{k\in\mathcal{K}} \; w_{(k)}} 
			+ 
			\Delta_{\alpha,2} ,
		\end{aligned}
	\end{equation}
	with $\Delta_{\alpha,2}$ is defined as 
	\begin{equation}\label{eq:del_2}
		\Delta_{\alpha,2} 
		= \left(\frac{\max_{k\in\mathcal{K}} \; w_{(k)}}{\min_{k\in\mathcal{K}} \; w_{(k)}}\right)^2 
		\cdot 
		\frac{32 N L^2  \Cd^2 \cdot \log\frac{2N}{\delta}  }{N }\cdot \alpha^2,
	\end{equation}
	where $\Cd$, and $\Ci$  are defined in Eq.~\eqref{eq:d_up}, and Eq.~\eqref{eq:Ci},  respectively.
	
	Furthermore, 
	\begin{equation}\label{eq:prob1}
		\Pr\left(\cap_{t=0}^{T-1}\; \cE_t\right) \geq 1 - T\cdot \left( \frac{(N+ 6)\delta}{2} + 2 \cdot \exp(- N \cdot D(1/4 \Vert p))\right), \mbox{ with } p = 0.0224,
	\end{equation}
	where the KL divergence  $D(q \Vert p) = q \log\frac{q}{p} + (1-q)\log\frac{1-q}{1-p}$.
\end{theorem}
\begin{proof}
Condition on $\cE_t$,  Lemma~\ref{lem:dec2} and Lemma~\ref{lem:del_up} hold. 
Accordingly, we have
\begin{align*}
&\frac{\min_{k\in\mathcal{K}} \; w_{(k)}}{\max_{k\in\mathcal{K}} \; w_{(k)}} \cdot \frac{\Ck}{16L\Ci\cdot \sqrt{2\log\frac{2N}{\delta}}} \norm{\nabla f(\x_t)}^2
\leq 
f(\x_t) - f(\x_{t+1}) 
+ \Delta_{\alpha,1}\\
\stackrel{\eqref{eq:del_up}}{\leq}& 
f(\x_t) - f(\x_{t+1})
+
T \cdot \Delta_{\alpha, 1}.
\end{align*}

Based on the condition that $\cE_t$ hold for $t = 0,\dots, T-1$ and summing above equation from $t =0$ to $T-1$, we can obtain that
\begin{align*}
&\frac{\min_{k\in\mathcal{K}} \; w_{(k)}}{\max_{k\in\mathcal{K}} \; w_{(k)}} 
\cdot 
\frac{\Ck}{16L\Ci\cdot \sqrt{2\log\frac{2N}{\delta}}}
\sum_{t=0}^{T-1} \norm{\nabla f(\x_t)}^2\\
\leq& 
f(\x_0) - f(\x_T)
+ 
T \cdot \Delta_{\alpha, 1}
\leq
f(\x_0) - f(\x^*)
+ 
T \cdot \Delta_{\alpha, 1}.
\end{align*}

Therefore, 
\begin{align*}
\frac{1}{T}\sum_{t=0}^{T-1} \norm{\nabla f(\x_t)}^2
\leq& 
\frac{f(\x_0) - f(\x^*)}{T}
\cdot 
\frac{16L\Ci\cdot \sqrt{2\log\frac{2N}{\delta}}}{\Ck}
\cdot 
\frac{\max_{k\in\mathcal{K}} \; w_{(k)}}{\min_{k\in\mathcal{K}} \; w_{(k)}} \\
&+
\frac{16L\Ci\cdot \sqrt{2\log\frac{2N}{\delta}}}{\Ck}
\cdot 
\frac{\max_{k\in\mathcal{K}} \; w_{(k)}}{\min_{k\in\mathcal{K}} \; w_{(k)}} \cdot \Delta_{\alpha, 1}\\
=&
\frac{f(\x_0) - f(\x^*)}{T}
\cdot 
\frac{32L\Ci\cdot \sqrt{2\log\frac{2N}{\delta}}}{N}
\cdot 
\frac{\max_{k\in\mathcal{K}} \; w_{(k)}}{\min_{k\in\mathcal{K}} \; w_{(k)}}
+ 
\Delta_{\alpha,2}.
\end{align*}

By the union bound and the probability that event $\cE_t$ holds shown in Eq.~\eqref{eq:prob}, we can obtain the result in Eq.~\eqref{eq:prob1}.
\end{proof}

Theorem~\ref{thm:main}  provides the convergence rate of Algorithm~\ref{alg:SA} when $f(\x)$ is both $L$-smooth and $\mu$-strongly convex.
Theorem~\ref{thm:main1} provides the convergence rate when $f(\x)$ is only $L$-smooth.
In the following two corollaries, we will provide the iteration complexities and query complexities of Algorithm~\ref{alg:SA}.
For convenience, we set $w_{(k)} = \frac{4}{N}$, that is, all chosen samples are set the same weight.
In fact, this same weight strategy is widely used \citep{hansen2016cma}.
In real applications,  the ratio $\frac{\max_{k\in\mathcal{K}} \; w_{(k)}}{\min_{k\in\mathcal{K}} \; w_{(k)}}$ is commonly a constant not far from $1$.
For example, the famous CMA-ES commonly sets a weight $w_{(k)} \propto \log (N+1) - \log(k)$. 
Accordingly, we have 
\begin{align*}
	\frac{\log (N+1) - \log (1)}{\log (N+1) - \log (N/4)} = \frac{\log(N+1)}{\log(4 (N+1)/N)} \approx \frac{\log (N)}{\log 4} \approx  2.16\; \mbox{ for } N = 20.
\end{align*}
Thus, we believe that $\frac{\max_{k\in\mathcal{K}} \; w_{(k)}}{\min_{k\in\mathcal{K}} \; w_{(k)}}$ can be regarded as a constant in the convergence and query analysis.

\begin{corollary}\label{cor:main}
	Let the objective function be $L$-smooth and $\mu$-strongly convex. 
	Given $0<\delta'<1$ and total iteration number $T$, set the sample size $N = \cO\left(\log\frac{T}{\delta'} + \log\log\frac{T}{\delta'}\right)$. 
	Assume the dimension $d$ is larger than $N$ and $\log\frac{TN}{\delta'}$.
	Set  $\delta = \cO\left(\frac{\delta'}{TN}\right)$ in Theorem~\ref{thm:main}. 
	Other parameters of Algorithm~\ref{alg:SA} set as Theorem~\ref{thm:main} and set $w_{(k)} = \frac{4}{N}$. 
	If the total iteration number $T$ satisfies that
	\begin{equation}\label{eq:T}
		T = \cO\left(\frac{dL}{\mu}\log\frac{1}{\varepsilon}\right), 
	\end{equation} 
	where $0<\varepsilon$, then with a probability at least $1-\delta'$, Algorithm~\ref{alg:SA} can find a point $\x_T$ satisfies that
	\begin{equation}\label{eq:pp}
		f(\x_T) - f(\x^*) 
		\leq 
		\varepsilon + \cO\left( \frac{d^2L^2 \left( \log\frac{dL}{\mu \delta'} + \log\log\frac{1}{\varepsilon} \right)}{\mu} \cdot \alpha^2 \right).
	\end{equation}
	and achieve a query complexity 
	\begin{equation}\label{eq:Q}
		Q = \cO\left( \frac{dL}{\mu} \cdot \log\frac{dL}{\mu \delta'} \cdot \log\frac{1}{\varepsilon} + \frac{dL}{\mu}\cdot \log\frac{1}{\varepsilon} \cdot \log\log\frac{1}{\varepsilon}\right).
	\end{equation}
\end{corollary}
\begin{proof}
By Eq.~\eqref{eq:prob1} with $\delta = \cO\left(\frac{\delta'}{TN}\right)$ and $N = \cO\left(\log\frac{T}{\delta'} + \log\log\frac{T}{\delta'}\right)$, we can obtain that 
\begin{align*}
	T\cdot \left( \frac{(N+ 6)\delta}{2} + 2 \cdot \exp(- N \cdot D(1/4 \Vert p))\right) 
	=
	\cO\left( \delta' + T \cdot \frac{\delta'}{T \log(T/\delta')} \right) 
	= 
	\cO(\delta').
\end{align*}
Then, Theorem~\ref{thm:main} holds for all $t=0,\dots, T-1$. 
Combining with Lemma~\ref{lem:dd}, we can obtain that
\begin{align*}
f(\x_{t+1}) - f(\x^*) - \frac{\Delta_{\alpha,1}'}{\rho} 
\leq 
\left( 1 - \rho \right) \left( f(\x_t) - f(\x^*) - \frac{\Delta_{\alpha,1}'}{\rho} \right),
\end{align*}
with $\rho$ defined as 
\begin{equation}\label{eq:rho}
	\rho = \frac{\min_{k\in\mathcal{K}} \; w_{(k)}}{\max_{k\in\mathcal{K}} \; w_{(k)}} 
	\cdot 
	\frac{\Ck}{8\Ci\cdot \sqrt{2\log\frac{2N}{\delta}}} 
	\cdot  
	\frac{\mu}{L}
	=
	\frac{\Ck}{8\Ci\cdot \sqrt{2\log\frac{2N}{\delta}}} 
	\cdot  
	\frac{\mu}{L}. 
\end{equation}
Using above equation recursively, we can obtain that
\begin{equation}\label{eq:dd}
\begin{aligned}
	f(\x_T) - f(\x^*) - \frac{\Delta_{\alpha,1}'}{\rho} 
	\leq& 
	(1-\rho)^T \left( f(\x_0) - f(\x^*) - \frac{\Delta_{\alpha,1}'}{\rho} \right)\\
	\leq&
	\exp(-\rho)^T \cdot \left( f(\x_0) - f(\x^*) - \frac{\Delta_{\alpha,1}'}{\rho} \right).
\end{aligned}
\end{equation}
Letting the right-hand side of above equation be $\varepsilon$, we only need $T$ to be
\begin{align*}
	T 
	= 
	\rho^{-1} \log\frac{1}{\varepsilon}
	=
	\cO\left( \frac{\Ci\cdot \sqrt{\log\frac{N}{\delta}}}{N} 
	\cdot  
	\frac{L}{\mu}  \log\frac{1}{\varepsilon} \right).
\end{align*} 
By $\delta = \frac{\delta'}{T N}$, $N = \cO\left(\log\frac{T}{\delta'} + \log\log\frac{T}{\delta'}\right)$, and definition of $\Ci$  in Eq.~\eqref{eq:Ci}, we can obtain that
\begin{align*}
\frac{\Ci\cdot \sqrt{\log\frac{N}{\delta}}}{N}
= 
\cO\left( \frac{d \sqrt{\log \frac{TN^2}{\delta'} } }{N} \right)
=
\cO\left( \frac{d \sqrt{\log\frac{T}{\delta'} + \log\log\frac{T}{\delta'} + \log\log\log \frac{T}{\delta'} } }{\log\frac{T}{\delta'} + \log\log\frac{T}{\delta'}}  \right)
=
\cO(d).
\end{align*} 
Thus, we only need $T$ to satisfy
\begin{align*}
	T = \cO\left(\frac{dL}{\mu}\log\frac{1}{\varepsilon}\right).
\end{align*}

The total query complexity $Q$ satisfies
\begin{align*}
	Q 
	=& TN 
	= \cO\left(T\log\frac{T}{\delta'}\right) 
	= \cO\left(\frac{dL}{\mu}\log\frac{1}{\varepsilon} \left( \log\frac{dL}{\mu \delta'} + \log\log\frac{1}{\varepsilon} \right)\right)\\
	=& 
	\cO\left( \frac{dL}{\mu} \cdot \log\frac{dL}{\mu \delta'} \cdot \log\frac{1}{\varepsilon} + \frac{dL}{\mu}\cdot \log\frac{1}{\varepsilon} \cdot \log\log\frac{1}{\varepsilon}\right).
\end{align*}

Furthermore, 
\begin{align*}
	\frac{\Delta_{\alpha,1}'}{\rho} 
	\stackrel{\eqref{eq:del_p}\eqref{eq:rho}}{=}&
	\cO\left( N L \Cd^2 \cdot \sqrt{\log\frac{N}{\delta}} \cdot \alpha^2 \cdot \frac{ \sqrt{\log\frac{N}{\delta}}}{\Ck} 
	\cdot  
	\frac{L}{\mu} \right)\\
	=&
	\cO\left( \frac{N L^2  \cdot d^2 \cdot \log\frac{N}{\delta} }{N} \cdot \frac{L}{\mu} \cdot \alpha^2  \right)\\
	=&\cO\left( \frac{d^2L^2 \left( \log\frac{dL}{\mu \delta'} + \log\log\frac{1}{\varepsilon} \right)}{\mu} \cdot \alpha^2 \right),
\end{align*}
where the second equality is because of $\Cd$, and  $\Ci$  defined in Eq.~\eqref{eq:d_up}, and  Eq.~\eqref{eq:Ci},  respectively.
The last equality is because of $\delta = \frac{\delta'}{T N}$ and $N = \cO\left(\log\frac{T}{\delta'} + \log\log\frac{T}{\delta'}\right)$.

Combining above equation with Eq.~\eqref{eq:dd}, we can obtain that
\begin{align*}
	f(\x_T) - f(\x^*) 
	\leq 
	\varepsilon + \frac{\Delta_{\alpha,1}'}{\rho} 
	= 
	\varepsilon + \cO\left( \frac{d^2L^2 \left( \log\frac{dL}{\mu \delta'} + \log\log\frac{1}{\varepsilon} \right)}{\mu} \cdot \alpha^2 \right),
\end{align*}
which proves Eq.~\eqref{eq:pp}.
\end{proof}

\begin{corollary}\label{cor:main1}
	Assume that the objective function $f(\x)$ is only $L$-smooth and its optimal value $f(\x^*)>-\infty$. 
		Given $0<\delta'<1$ and total iteration number $T$, set the sample size $N = \cO\left(\log\frac{T}{\delta'} + \log\log\frac{T}{\delta'}\right)$. 
		Assume the dimension $d$ is larger than $N$ and $\log\frac{TN}{\delta'}$.
	Set  $\delta = \cO\left(\frac{\delta'}{TN}\right)$ in Theorem~\ref{thm:main1}.
		Other parameters of Algorithm~\ref{alg:SA} set as Theorem~\ref{thm:main1} and set $w_{(k)} = \frac{4}{N}$. 
	If the total iteration number $T$ satisfies that
	\begin{equation}\label{eq:T1}
		T = \cO\left(\frac{dL}{\varepsilon}\right), 
	\end{equation} 
	where $0<\varepsilon$, then with a probability at least $1-\delta'$, Algorithm~\ref{alg:SA} can find a point $\x_T$ satisfies that
	\begin{equation}\label{eq:pp1}
		f(\x_T) - f(\x^*) 
		\leq 
		\varepsilon + \cO\left(d^2L^2\log\left(\frac{dL}{\varepsilon\delta'}\right) \cdot \alpha^2\right).
	\end{equation}
	and achieve a query complexity 
	\begin{equation}\label{eq:Q1}
		Q = \cO\left( \frac{dL}{\varepsilon} \log\frac{dL}{\delta'\varepsilon} \right).
	\end{equation}
\end{corollary}
\begin{proof}
By Eq.~\eqref{eq:prob1} with $\delta = \cO\left(\frac{\delta'}{TN}\right)$ and $N = \cO\left(\log\frac{T}{\delta'} + \log\log\frac{T}{\delta'}\right)$, we can obtain that 
\begin{align*}
	T\cdot \left( \frac{(N+ 6)\delta}{2} + 2 \cdot \exp(- N \cdot D(1/4 \Vert p))\right) 
	=
	\cO\left( \delta' + T \cdot \frac{\delta'}{T \log(T/\delta')} \right) 
	= 
	\cO(\delta').
\end{align*}
Thus, results in Theorem~\ref{thm:main1} hold.
By setting the first term of right-hand of Eq.~\eqref{eq:main_sub} to $\varepsilon$, we can obtain that
\begin{align*}
	T =& 
	\frac{f(\x_0) - f(\x^*)}{\varepsilon}
	\cdot 
	\frac{16L\Ci\cdot \sqrt{2\log\frac{2N}{\delta}}}{\Ck}\\
	=&
	\cO\left( \frac{L}{\varepsilon} \cdot \frac{\Ci\cdot \sqrt{\log\frac{N}{\delta}}}{N}  \right)
	=
	\cO\left( \frac{L}{\varepsilon} \cdot \frac{\Ci\cdot \sqrt{\log\frac{N}{\delta}}}{N}  \right)\\
	=&
	\cO\left(\frac{L}{\varepsilon}\cdot\frac{d\cdot\sqrt{\log\frac{T}{\delta'} + \log\log\frac{T}{\delta'} + \log\log\log \frac{T}{\delta'} }}{\sqrt{\log\frac{T}{\delta'} + \log\log\frac{T}{\delta'}  }}\right)\\
	=&
	\cO\left(\frac{d L}{\varepsilon}\right).
\end{align*}
Accordingly, the query complexity is 
\begin{align*}
	Q = TN = \cO\left( \frac{dL}{\varepsilon} \log\frac{dL}{\delta'\varepsilon} \right).
\end{align*}

By the definition of $\Delta_{\alpha,2}$, and  $\Cd$, and $\Ci$ defined in Eq.~\eqref{eq:d_up}, and Eq.~\eqref{eq:Ci}, with $\delta = \frac{\delta'}{TN}$, we have
\begin{align*}
	\Delta_{\alpha,2} 
	\stackrel{\eqref{eq:del_2}}{=}& 
	\frac{32 N L^2  \Cd^2 \cdot \log\frac{2N}{\delta}  }{N }\cdot \alpha^2\\
	=&
	\cO\left(  L^2 d^2 \log\frac{TN}{\delta'} \cdot \alpha^2 \right)
	=
	\cO\left(d^2L^2\log\left(\frac{dL}{\varepsilon\delta'}\right) \cdot \alpha^2\right)
\end{align*}
\end{proof}

\begin{remark}
Corollary~\ref{cor:main} and Corollary~\ref{cor:main1} provide explicit and non-asymptotic query complexities of our rank-based zeroth-order algorithm for smooth functions. 
To my best knowledge, these  query complexities are first proved.
Thus, we believe our work provides the first \emph{explicit} and \emph{non-asymptotic} query complexities of rank-based ZO algorithms.
\end{remark}

\subsection{Discussion}
\label{subsec:discuss}

The convergence analysis in previous subsections provide some interesting and novel insights on the rank-based ZO algorithms. 
In this section, we will give an detailed discussion. 

\paragraph{Rank-based versus Value-based.}
When $f(\x)$ is $L$-smooth and $\mu$-strongly convex, Corollary~\ref{cor:main} shows that Algorithm~\ref{alg:SA} achieve a query complexity
\begin{align*}
	Q = \cO\left( \frac{dL}{\mu} \cdot \log\frac{dL}{\mu \delta'} \cdot \log\frac{1}{\varepsilon} + \frac{dL}{\mu}\cdot \log\frac{1}{\varepsilon} \cdot \log\log\frac{1}{\varepsilon}\right),
\end{align*}  
which is almost the same to the query complexity 
\begin{align*}
	\cO\left(\frac{dL}{\mu}\log\frac{1}{\varepsilon}\right),
\end{align*} 
achieved by the value based zeroth-order algorithm \citep{nesterov2017random}.
However, our complexity holds with a high probability at least $1 - \delta'$. 
In contrast, the complexity in \citet{nesterov2017random} holds in expectation. 
Except the difference in $\log\frac{1}{\delta'}$ term, we can observe that our rank-based algorithm requires at most $\log\log\frac{1}{\varepsilon}$ more queries than valued based zeroth-order algorithm.

Similarly, when  $f(\x)$ is $L$-smooth and may be nonconvex, Corollary~\ref{cor:main1} shows that our rank-based zeroth-order algorithm  achieves a query complexity $$	Q = \cO\left( \frac{dL}{\varepsilon} \log\frac{dL}{\delta'\varepsilon} \right).$$
This complexity is at most $\cO(\log\frac{1}{\varepsilon})$ larger than the value based zeroth-order algorithm \citep{ghadimi2013stochastic}.

Compared the query complexities of our rank-based zeroth-order algorithm with value-based ones, we can conclude that the rank-based zeroth-order algorithm can achieve almost the same query complexities to the valued-based counterparts. 
This is an interesting theory result since   only the relative ranking is used in a rank-based zeroth-order algorithm which valued-based algorithms can access the function value. 
That is, valued-based algorithms use more information than rank-based zeroth-order algorithms but achieve almost the same query complexities.
It is natural to ask that whether the relative ranking is sufficient for the zeroth-order algorithm or the valued-based zeroth-order algorithms can achieve a lower query complexity.

\paragraph{How to choose weights and why log weights?}

Convergence theories in Eq.~\eqref{eq:main_lin} and Eq.~\eqref{eq:main_sub} show that  setting all weights $w_{(k)}$ equally is the best strategy. 
However, in practice, it is commonly to set a weight of $\ub_i$ according to its rank $f(\x+\alpha \ub_i)$.
Intuitively, one should set a larger weight on $\ub_{k_1}$ over $\ub_{k_2}$ in Algorithm~\ref{alg:SA}. That is, set $w_{k_1}^+ > w_{k_2}^+ > 0$,  if $ (k_1) < (k_2)\leq \frac{N}{4}$. 
For example, the famous CMA-ES set weights $w_{(k)} \propto \log(N+1) - \log(k) $ \citep{hansen2001completely}. 
Our theory in Theorem~\ref{thm:main} and Theorem~\ref{thm:main1} can not explain the success of weight strategy of CMA-ES.  

In fact,  Eq.~\eqref{eq:main_lin} and Eq.~\eqref{eq:main_sub} over-relax the result in Eq.~\eqref{eq:dec1}. 
With a proper weight strategy, one may obtain 
\begin{align*}
	\min_{k\in\mathcal{K}} \left|\frac{\dotprod{\nabla f(\x_t), \ub_{(k)}}}{w_{(k)}}\right|
	\cdot \min_{k\in\mathcal{K}} \left|\frac{w_{(k)}}{\dotprod{\nabla f(\x_t), \ub_{(k)}}}\right| 
	> 
	\frac{\min_{k\in\mathcal{K}} \; w_{(k)}}{\max_{k\in\mathcal{K}} \; w_{(k)}} 
	\cdot \frac{\min_{k\in\mathcal{K}} \; |\dotprod{\nabla f(\x_t), \ub_{(k)}}|}{\max_{k\in\mathcal{K}} \; |\dotprod{\nabla f(\x_t), \ub_{(k)}}|},
\end{align*}
which implies a faster convergence.

However, the gain from a proper weight strategy has its upper bound over setting all weights $w_{(k)}$ equally at most $\sqrt{2\log\frac{2N}{\delta}}$. 
First, notice that the fact that it holds that
\begin{align*}
	\min_{k\in\mathcal{K}} \left|\frac{\dotprod{\nabla f(\x_t), \ub_{(k)}}}{w_{(k)}}\right|
	\cdot \min_{k\in\mathcal{K}} \left|\frac{w_{(k)}}{\dotprod{\nabla f(\x_t), \ub_{(k)}}}\right| \leq 1.
\end{align*}
Second, Eq.~\eqref{eq:dec2} gives a lower bound if $w_{(k)}$'s are set equally
\begin{align*}
\min_{k\in\mathcal{K}} \left|\frac{\dotprod{\nabla f(\x_t), \ub_{(k)}}}{w_{(k)}}\right|
\cdot \min_{k\in\mathcal{K}} \left|\frac{w_{(k)}}{\dotprod{\nabla f(\x_t), \ub_{(k)}}}\right|
\geq 
\frac{1}{\sqrt{2\log\frac{2N}{\delta}}}.
\end{align*}

Next, we will try to set proper weights $w_{(k)}$ to achieve a large descent value:
\begin{equation}\label{eq:ww}
	\max_{w_{(k)}} \left( \min_{k\in\mathcal{K}} \left|\frac{\dotprod{\nabla f(\x_t), \ub_{(k)}}}{w_{(k)}}\right|
	\cdot \min_{k\in\mathcal{K}} \left|\frac{w_{(k)}}{\dotprod{\nabla f(\x_t), \ub_{(k)}}}\right|\right)
\end{equation}
In fact, $\dotprod{\nabla f(\x_t), \ub_{(k)}} $ shares the same distribution $\norm{\nabla f(\x_t)} \cdot \{X_{(1)}, \dots, X_{(N)} \}$ where $X_{(i)}$ is the $i$-th order statistics of $N$ i.i.d.\ Gaussian random variable, that is, $X_i \sim \cN(0,1)$. 
Accordingly, we have  $\dotprod{\nabla f(\x_t), \ub_{(k)}} = \norm{\nabla f(\x_t)} | X_{(k)} |$. 
To minimize Eq.~\eqref{eq:ww} by setting proper $w_{(k)}$, a good option is to set  $w_{(k)} \propto \EE [X_i]$. 
A widely used approximation for the expected value of the $k$-th order statistic is Blom's formula \citep{blom1958statistical}:
\begin{equation}\label{eq:blom}
	\mathbb{E}[X_{(k)}]
	\approx
	\Phi^{-1}\!\left(
	\frac{k - 0.375}{\,N + 0.25\,}
	\right),
\end{equation}
where $\Phi^{-1}(\cdot)$ is the inverse cumulative distribution function of the Gaussian distribution.

It is interesting that the weight $w_{(k)}^+ \propto
\Phi^{-1}\!\left(
\frac{k - 0.375}{\,N + 0.25\,}
\right)$
is very close to the one $w_{(k)}^+ \propto \log(N+1) - \log(k)$ widely used in CMA-ES \citep{hansen2001completely} and natural evolution strategies \citep{wierstra2014natural}.
Thus, we conjecture that the success of ``log weights'' strategy lies in the approximation of the weights in proportion to Eq.~\eqref{eq:blom}.

\paragraph{Can ``negative'' sample help to achieve faster convergence?}

CMA-ES commonly give weights over directions achieve on the $\lambda$ smallest values where $\lambda = \frac{N}{4}$ in this paper. 
Instead, natural evolution strategies commonly exploit ``negative'' directions and set ``negative'' weight accordingly \citep{glasmachers2010exponential,wierstra2014natural}. 
However, the CMA-ES abandons the information in ``negative'' directions to update $\x_t$ \citep{hansen2001completely}. 
Our Algorithm~\ref{alg:SA}  also set ``negative'' weight over those ``negative'' directions which achieve the $\lambda$ largest values. 
Next, we will discuss how these ``negative'' directions help to achieve faster convergence.

Assume that we only use $w_{(k)}^+$, Lemma~\ref{lem:dec1} will become 
\begin{align*}
	f(\x_{t+1}) 
	\leq& f(\x_t) 
	- \frac{1}{16L\Ci'} \cdot\min_{k\in\mathcal{K^+}} \left|\frac{\dotprod{\nabla f(\x_t), \ub_{(k)}}}{w_{(k)}}\right|
	\cdot \min_{k\in\mathcal{K^+}} \left|\frac{w_{(k)}}{\dotprod{\nabla f(\x_t), \ub_{(k)}}}\right|\\
	&
	\cdot  \sum_{k\in \mathcal{K^+}} \dotprod{\nabla f(\x_t), \ub_{(k)}}^2 + \Delta_{\alpha,1},
\end{align*}
where we $\mathcal{K^+} = \{1,\dots, N/4\}$ and $\Ci' =\left( \sqrt{N/4} + \sqrt{d} + \sqrt{2\log\frac{2}{\delta}} \right)^2$.
Since $\mathcal{K^+}$ only has half samples of $\mathcal{K}$ and $\Ci' \approx \Ci = \left( \sqrt{N/2} + \sqrt{d} + \sqrt{2\log\frac{2}{\delta}} \right)^2 $ since it holds  $N \ll d$,
Comparing above equation to Eq.~\eqref{eq:dec1}, we can concludes that Algorithm~\ref{alg:SA} can be almost twice fast as the algorithm without ``negative'' directions.
Just as above equations shown,  the ``negative'' directions should be fully exploited and they will help to achieve a twice faster convergence rate.

\section{Conclusion}
This paper establishes the first explicit convergence rates for rank-based zeroth-order optimization, addressing a long-standing theoretical open problem. 
By analyzing a simple  rank-based zeroth-order method, we obtain the non-asymptotic query complexities under standard smoothness assumptions. Our proof framework departs from traditional drift and information-geometric analyses, revealing structural principles behind ranking-based updates. These findings not only clarify the behavior of rank-based ZO methods but also provide tools that may facilitate the analysis of more advanced evolution strategies in future work.

\pb
\clearpage
\bibliography{ref.bib}

\begin{thebibliography}{}

\bibitem[Akimoto et~al., 2022]{akimoto2022global}
Akimoto, Y., Auger, A., Glasmachers, T., \& Morinaga, D. (2022).
\newblock Global linear convergence of evolution strategies on more than smooth
  strongly convex functions.
\newblock {\em SIAM Journal on Optimization}, 32(2), 1402--1429.

\bibitem[Akimoto et~al., 2012]{akimoto2012theoretical}
Akimoto, Y., Nagata, Y., Ono, I., \& Kobayashi, S. (2012).
\newblock Theoretical foundation for cma-es from information geometry
  perspective.
\newblock {\em Algorithmica}, 64(4), 698--716.

\bibitem[B{\"a}ck et~al., 2023]{back2023evolutionary}
B{\"a}ck, T.~H., Kononova, A.~V., van Stein, B., Wang, H., Antonov, K.~A.,
  Kalkreuth, R.~T., de~Nobel, J., Vermetten, D., de~Winter, R., \& Ye, F.
  (2023).
\newblock Evolutionary algorithms for parameter optimization—thirty years
  later.
\newblock {\em Evolutionary Computation}, 31(2), 81--122.

\bibitem[Beyer \& Arnold, 2001]{beyer2001theory}
Beyer, H.-G. \& Arnold, D.~V. (2001).
\newblock Theory of evolution strategies—a tutorial.
\newblock {\em Theoretical aspects of evolutionary computing}, (pp.\ 109--133).

\bibitem[Blom, 1958]{blom1958statistical}
Blom, G. (1958).
\newblock {\em Statistical estimates and transformed beta-variables}.
\newblock PhD thesis, Almqvist \& Wiksell.

\bibitem[Chrabaszcz et~al., 2018]{chrabaszcz2018back}
Chrabaszcz, P., Loshchilov, I., \& Hutter, F. (2018).
\newblock Back to basics: Benchmarking canonical evolution strategies for
  playing atari.
\newblock {\em arXiv preprint arXiv:1802.08842}.

\bibitem[Chvátal, 1979]{CHVATAL1979285}
Chvátal, V. (1979).
\newblock The tail of the hypergeometric distribution.
\newblock {\em Discrete Mathematics}, 25(3), 285--287.

\bibitem[Ghadimi \& Lan, 2013]{ghadimi2013stochastic}
Ghadimi, S. \& Lan, G. (2013).
\newblock Stochastic first-and zeroth-order methods for nonconvex stochastic
  programming.
\newblock {\em SIAM journal on optimization}, 23(4), 2341--2368.

\bibitem[Glasmachers et~al., 2010]{glasmachers2010exponential}
Glasmachers, T., Schaul, T., Yi, S., Wierstra, D., \& Schmidhuber, J. (2010).
\newblock Exponential natural evolution strategies.
\newblock In {\em Proceedings of the 12th annual conference on Genetic and
  evolutionary computation}  (pp.\ 393--400).

\bibitem[Hansen, 2016]{hansen2016cma}
Hansen, N. (2016).
\newblock The cma evolution strategy: A tutorial.
\newblock {\em arXiv preprint arXiv:1604.00772}.

\bibitem[Hansen \& Ostermeier, 2001]{hansen2001completely}
Hansen, N. \& Ostermeier, A. (2001).
\newblock Completely derandomized self-adaptation in evolution strategies.
\newblock {\em Evolutionary computation}, 9(2), 159--195.

\bibitem[Hasenj{\"a}ger et~al., 2005]{hasenjager2005three}
Hasenj{\"a}ger, M., Sendhoff, B., Sonoda, T., \& Arima, T. (2005).
\newblock Three dimensional evolutionary aerodynamic design optimization with
  cma-es.
\newblock In {\em Proceedings of the 7th annual conference on Genetic and
  evolutionary computation}  (pp.\ 2173--2180).

\bibitem[He \& Yao, 2001]{he2001drift}
He, J. \& Yao, X. (2001).
\newblock Drift analysis and average time complexity of evolutionary
  algorithms.
\newblock {\em Artificial intelligence}, 127(1), 57--85.

\bibitem[Igel, 2003]{igel2003neuroevolution}
Igel, C. (2003).
\newblock Neuroevolution for reinforcement learning using evolution strategies.
\newblock In {\em The 2003 Congress on Evolutionary Computation, 2003.
  CEC'03.}, volume~4  (pp.\ 2588--2595).: IEEE.

\bibitem[Laurent \& Massart, 2000]{laurent2000adaptive}
Laurent, B. \& Massart, P. (2000).
\newblock Adaptive estimation of a quadratic functional by model selection.
\newblock {\em Annals of Statistics}, 28(5), 1302--1338.

\bibitem[Loshchilov \& Hutter, 2016]{loshchilov2016cma}
Loshchilov, I. \& Hutter, F. (2016).
\newblock Cma-es for hyperparameter optimization of deep neural networks.
\newblock {\em arXiv preprint arXiv:1604.07269}.

\bibitem[Morinaga et~al., 2021]{morinaga2021convergence}
Morinaga, D., Fukuchi, K., Sakuma, J., \& Akimoto, Y. (2021).
\newblock Convergence rate of the (1+ 1)-evolution strategy with success-based
  step-size adaptation on convex quadratic functions.
\newblock In {\em Proceedings of the Genetic and Evolutionary Computation
  Conference}  (pp.\ 1169--1177).

\bibitem[Nesterov, 2013]{nesterov2013introductory}
Nesterov, Y. (2013).
\newblock {\em Introductory lectures on convex optimization: A basic course},
  volume~87.
\newblock Springer Science \& Business Media.

\bibitem[Nesterov \& Spokoiny, 2017]{nesterov2017random}
Nesterov, Y. \& Spokoiny, V. (2017).
\newblock Random gradient-free minimization of convex functions.
\newblock {\em Foundations of Computational Mathematics}, 17(2), 527--566.

\bibitem[Ollivier et~al., 2017]{ollivier2017information}
Ollivier, Y., Arnold, L., Auger, A., \& Hansen, N. (2017).
\newblock Information-geometric optimization algorithms: A unifying picture via
  invariance principles.
\newblock {\em Journal of Machine Learning Research}, 18(18), 1--65.

\bibitem[Rechenberg, 1978]{rechenberg1978evolutionsstrategien}
Rechenberg, I. (1978).
\newblock Evolutionsstrategien.
\newblock In {\em Simulationsmethoden in der Medizin und Biologie: Workshop,
  Hannover, 29. Sept.--1. Okt. 1977}  (pp.\ 83--114).: Springer.

\bibitem[Such et~al., 2017]{such2017deep}
Such, F.~P., Madhavan, V., Conti, E., Lehman, J., Stanley, K.~O., \& Clune, J.
  (2017).
\newblock Deep neuroevolution: Genetic algorithms are a competitive alternative
  for training deep neural networks for reinforcement learning.
\newblock {\em arXiv preprint arXiv:1712.06567}.

\bibitem[Vershynin, 2018]{vershynin2018high}
Vershynin, R. (2018).
\newblock {\em High-Dimensional Probability: An Introduction with Applications
  in Data Science}.
\newblock Cambridge University Press.

\bibitem[Whitley, 1989]{whitley1989genitor}
Whitley, D. (1989).
\newblock The genitor algorithm and selection pressure: Why rank-based
  allocation of reproductive trials is best.
\newblock In {\em Proceedings of the third international conference on Genetic
  algorithms}.

\bibitem[Wierstra et~al., 2014]{wierstra2014natural}
Wierstra, D., Schaul, T., Glasmachers, T., Sun, Y., Peters, J., \& Schmidhuber,
  J. (2014).
\newblock Natural evolution strategies.
\newblock {\em The Journal of Machine Learning Research}, 15(1), 949--980.

\bibitem[Ye et~al., 2025]{ye2025unified}
Ye, H., Chang, X., \& Chen, X. (2025).
\newblock A unified zeroth-order optimization framework via oblivious
  randomized sketching.
\newblock {\em arXiv preprint arXiv:2510.10945}.

\end{thebibliography}
\bibliographystyle{apalike2}

\appendix

\section{Probability Tools}

\begin{lemma}[equivalence to a binomial event]\label{lem:bin-equiv}
	Let \(X_1,\dots,X_N\) be i.i.d.\ continuous random variables with cumulative distribution function $\Phi(\cdot)$.
	Denote their order statistics by $X_{(1)} \leq X_{(2)} \leq \dots \leq X_{(N)}$.
	Fix \(\tau\in\mathbb{R}\) and put \(m\in\{1,\dots,N\}\).
	Define
	\[
	M:=X_{(N-m+1)}.
	\]
	Then
	\[
	\Pr(M>\tau) = \Pr\Big(\sum_{i=1}^N \mathbf{1}_{\{X_i>\tau\}} \ge m\Big) = \Pr\big(\mathrm{Bin}(N,p)\ge m\big),
	\]
    where $\mathrm{Bin}(N,p)$ is the Binomial distribution with parameters $N$ and $p:=\Pr(X_i>\tau)=1-\Phi(\tau)$.
\end{lemma}

\begin{proof}
	The event \(\{M>\tau\}\) means that the \((N-m+1)\)-th smallest observation exceeds \(\tau\).
	Equivalently, at least \(m\) of the \(N\) observations exceed \(\tau\). This is exactly the event
	\(\{\sum_{i=1}^N \mathbf{1}_{\{X_i>\tau\}} \ge m\}\). Since the indicators \(\mathbf{1}_{\{X_i>\tau\}}\) are i.i.d.\ Bernoulli\((p)\)
	with \(p=1-F(\tau)\), the sum has distribution \(\mathrm{Bin}(N,p)\), which proves the lemma.
\end{proof}

\begin{lemma}[\cite{CHVATAL1979285}]\label{lem:chenoff}
Let $X_i$ be independent Bernoulli random variables following the same distribution
as random variable $X$, $X \sim B(p)$, where $0 < p < 1$. Let $S_N = \sum_i X_i$. Then, for any $r$ such that
$p < r < 1$ we have the following
\begin{equation}
	\Pr(S_N \ge r N) \leq \exp(-N \cdot D(r || p)), \mbox{ where } D(r || p) = r\ln\frac{r}{p} + (1-r)\ln\frac{1-r}{1-p}.
\end{equation}
\end{lemma}

\begin{lemma}\label{lem:low1}
Let \(X_1,\dots,X_N\) be i.i.d.\ standard Gaussian,
\(m=\lfloor N/4\rfloor\), and \(M:=X_{(N-m+1)}\), where $X_{(N-m+1)}$ means the $(N-m+1)$-th largest in $X_i$'s.
Given a fixed value $\tau$ and denoting $p=1-\Phi(\tau)$,   if \(p< q :=\tfrac{m}{N}\), then the Chernoff bound yields the exponential lower bound
\begin{equation}\label{eq:chernoff-lower}
	\Pr(M>\tau)
	\;\ge\; 1-\exp\!\big(-N\,D(q \Vert p)\big),
\end{equation}
where the binary Kullback--Leibler divergence is
\begin{equation}\label{eq:KL}
D(q \Vert p) \;=\; \alpha\ln\frac{q}{p}+(1-q)\ln\frac{1-q}{1-p}.
\end{equation}
In particular, for \(\tau=2\) we have \(p\approx 0.0224 < 1/4\) and hence (with \(q= 1/4\))
\[
\Pr(M>2) \;\ge\; 1-\exp\big(-N\,D(1/4\Vert p)\big).
\]
\end{lemma}
\begin{proof}
First,  Lemma \ref{lem:bin-equiv} shows that $\Pr(M>\tau)=\Pr\big(\mathrm{Bin}(N,p)\ge m\big)$.
Furthermore, $\mathrm{Bin}(N,p)$ equals to $S_N = \sum_{i=1}^N X_i'$ where $X_i'$ are i.i.d.\ independent Bernoulli random variables following the same distribution as random variable $X'$, $X' \sim B(p)$.
By Lemma~\ref{lem:chenoff} with $r = \frac{1}{4}$, we can obtain the result.

\end{proof}

\begin{lemma}[equivalence to a binomial event]\label{lem:bin-equiv1}
	Let \(X_1,\dots,X_N\) be i.i.d.\ continuous random variables with cumulative distribution function $\Phi(\cdot)$.
	Denote their order statistics by $X_{(1)} \leq X_{(2)} \leq \dots \leq X_{(N)}$.
	Fix \(\tau\in\mathbb{R}\) and put \(m\in\{1,\dots,N\}\).
	Define
	\[
	M:=X_{(m)}.
	\]
	Then
	\[
	\Pr(M<\tau) = \Pr\Big(\sum_{i=1}^N \mathbf{1}_{\{X_i<\tau\}} \ge m\Big) = \Pr\big(\mathrm{Bin}(N,p)\ge m\big),
	\]
	where $\mathrm{Bin}(N,p)$ is the Binomial distribution with parameters $N$ and $p:=\Pr(X_i<\tau)=\Phi(\tau)$.
\end{lemma}

\begin{proof}
	The event \(\{M<\tau\}\) means that the \(m\)-th smallest observation less than \(\tau\).
	Equivalently, at least \(m\) of the \(N\) observations less than \(\tau\). This is exactly the event
	\(\{\sum_{i=1}^N \mathbf{1}_{\{X_i<\tau\}} \ge m\}\). Since the indicators \(\mathbf{1}_{\{X_i>\tau\}}\) are i.i.d.\ Bernoulli\((p)\)
	with \(p=F(\tau)\), the sum has distribution \(\mathrm{Bin}(N,p)\), which proves the lemma.
\end{proof}

\begin{lemma}\label{lem:low2}
	Let \(X_1,\dots,X_N\) be i.i.d.\ standard Gaussian,
	\(m=\lfloor N/4\rfloor\), and \(M:=X_{(m)}\), where $X_{(m)}$ means the $m$-th largest in $X_i$'s.
	Given a fixed value $\tau$, and denoting  $p=\Phi(\tau)$, if \(p < q :=\tfrac{m}{N} \), then the Chernoff bound yields the exponential lower bound
	\begin{equation}\label{eq:chernoff-lower1}
		\Pr(M< \tau )
		\;\ge\; 1-\exp\!\big(-N\,D(q \Vert p)\big),
	\end{equation}
	where the binary Kullback--Leibler divergence $D(q \Vert p)$ is defined in Eq.~\eqref{eq:KL}.
	In particular, for \(\tau=-2\) we have \(p\approx 0.0224 <1/4 \) and hence 
	\[
	\Pr(M<-2) \;\ge\; 1-\exp\big(-N\,D(1/4\Vert p)\big).
	\]
\end{lemma}
\begin{proof}
	 Lemma \ref{lem:bin-equiv1} shows that $\Pr(M<\tau)=\Pr\big(\mathrm{Bin}(N,p)\ge m\big)$.
	Furthermore, $\mathrm{Bin}(N,p)$ equals to $S_N = \sum_{i=1}^N X_i'$ where $X_i'$ are i.i.d.\ independent Bernoulli random variables following the same distribution as random variable $X'$, $X' \sim B(p)$.
	By Lemma~\ref{lem:chenoff} with $r = 1/4$, we can obtain the result.
\end{proof}

\begin{lemma}[Gaussian Tail Bound \citep{vershynin2018high,laurent2000adaptive}]
	\label{lem:gaussian-tail}
	Let $X\sim\mathcal N(0,1)$. Then for any $\tau>0$,
	\[
	\Pr(|X|>\tau) \le 2 \exp\!\left( -\frac{\tau^2}{2} \right).
	\]
\end{lemma}

\begin{proof}
	By the union bound,
	\[
	\Pr(M_N>t)
	= \Pr\!\left( \bigcup_{i=1}^N \{|X_i|>t\} \right)
	\le \sum_{i=1}^N \Pr(|X_i|>t).
	\]
	Applying Lemma~\ref{lem:gaussian-tail} completes the proof.
\end{proof}

\begin{theorem}
	\label{thm:gauss_up}
	Let $X_1,\dots,X_N$ be i.i.d.\ $\mathcal N(0,1)$ variables and 
	$M_N=\max_{i\le N}|X_i|$.  Then 
	\[
	\Pr\!\left( M_N \le \sqrt{2\log\frac{2N}{\delta}} \right)
	\ge 1 - \delta, \mbox{ with } 0<\delta<1.
	\]
\end{theorem}

\begin{lemma}[$\chi^2$ tail bound~\cite{laurent2000adaptive}]
	\label{lem:chi}
	Let $q_1, \dots, q_n$ be independent $\chi^2$ random variables, each with one degree of freedom. For any vector $\gamma = (\gamma_1, \dots , \gamma_n) \in \RR_+^n$ with non-negative
	entries, and any $\tau > 0$,
	\begin{align*}
	\Pr\left[\sum_{i=1}^{n}\gamma_iq_i \geq \norm{\gamma}_1 + 2 \sqrt{\norm{\gamma}^2 \tau} + 2\norm{\gamma}_\infty \tau\right] \leq \exp(-\tau),
	\end{align*}
	where $\norm{\gamma}_1 = \sum_{i=1}^{n} |\gamma_i|$ and $\norm{\gamma}_\infty = \max\{|\gamma_i|\}$.
\end{lemma} 

\begin{lemma}[\cite{vershynin2018high}]\label{lem:U_up}
	Let $A$ be an $N \times n$ matrix whose entries are independent standard normal random variables. Then for every $\tau \ge 0$,
	with probability at least $1 - 2 \exp(-\tau^2/2)$, the largest singular value $s_{\max}(A)$ satisfies
	\begin{equation}
		\label{eq:smax}
		s_{\max}(A) \le \sqrt{N} + \sqrt{n} + \tau. 
	\end{equation}
\end{lemma}

\begin{lemma}\label{lem:u_up}
	Letting $\ub\sim N(0, \bm{I}_d)$ be a $d$-dimensional Gaussian vector, then with probability at least $1-\delta$, it  holds that
	\begin{equation}
		\label{eq:u_norm}
		\norm{\ub}^2 \le 2d + 3\log(1/\delta).
	\end{equation}
\end{lemma}
\begin{proof}
	By Lemma~\ref{lem:chi}, we have
	\begin{equation*}
		\norm{u}^2 \le d + 2\sqrt{d\log(1/\delta)} + 2\log(1/\delta) \le 2d + 3\log(1/\delta).
	\end{equation*}
\end{proof}

\section{Useful Lemmas}

\begin{lemma}[Theorem 2.1.10 of \citet{nesterov2013introductory}]\label{lem:str_cvx}
Letting a function $f(\x)$ is differentiable and $\mu$-strongly convex, then it holds that
\begin{equation*}
	\norm{\nabla f(\x)}^2 \geq 2\mu \Big( f(\x) - f(\x^*)\Big).
\end{equation*}
\end{lemma}

\begin{lemma}
	\label{lem:dd}
	Letting non-negative sequence $\{\Delta_t\}$ satisfy 
	$\Delta_{t+1} \leq (1 - \beta) \Delta_t + c$ with $0\leq \beta \leq 1$ and $c\geq 0$, then it holds that $\Delta_{t+1} - \frac{c}{\beta} \leq (1-\beta)\left(\Delta_t - \frac{c}{\beta}\right)$.
\end{lemma}
\begin{proof}
	We assume that $\{\Delta_t\}$ satisfies that $ \Delta_{t+1} - c' \leq \left(1 - \beta\right)\left(\Delta_t - c'\right) $ with $c'\geq 0$. 
	Then, it implies that 
	\begin{equation*}
		\Delta_{t+1} \leq  (1 - \beta) \Delta_t + c' - (1-\beta)c' = (1 - \beta) \Delta_t + \beta c'.
	\end{equation*}
	Thus, we can obtain that $c' = \beta^{-1} c$ which concludes the proof.
\end{proof}

%
\end{document}